\documentclass[twoside]{article}

\usepackage[accepted]{aistats2022}

\usepackage[round]{natbib}

\usepackage[utf8]{inputenc} %
\usepackage[T1]{fontenc}    %
\usepackage{hyperref}       %
\usepackage{url}            %
\usepackage{booktabs}       %
\usepackage{amsfonts}       %
\usepackage{nicefrac}       %
\usepackage{microtype}      %
\usepackage{xcolor}         %
\usepackage{microtype}
\usepackage{graphicx}
\usepackage{subcaption}
\usepackage{booktabs} %
\usepackage{bbold}
\usepackage{amsmath}
\usepackage{amsthm}
\usepackage{algorithm}
\usepackage{algorithmic}
\usepackage{hyperref}
\usepackage{multirow}
\usepackage{enumitem}
\usepackage{nameref}
\usepackage[normalem]{ulem}

\usepackage{comment}

\usepackage{color-edits}
\addauthor{ra}{red}
\addauthor{jl}{blue}
\addauthor{pf}{orange}

\makeatletter
\newtheorem*{rep@theorem}{\rep@title}
\newcommand{\newreptheorem}[2]{%
\newenvironment{rep#1}[1]{%
 \def\rep@title{#2 \ref{##1}}%
 \begin{rep@theorem}}%
 {\end{rep@theorem}}}
\makeatother

\theoremstyle{plain}
\newtheorem{theorem}{Theorem}

\newtheorem{lemma}{Lemma}

\theoremstyle{definition}
\newtheorem{definition}{Definition}

\theoremstyle{remark}

\newreptheorem{theorem}{Theorem}
\newreptheorem{lemma}{Lemma}
\newreptheorem{definition}{Definition}

\newcommand{\R}{\mathbb{R}}
\newcommand{\E}{\mathbb{E}}
\newcommand{\prob}{\mathbb{P}}
\newcommand{\X}{\mathcal{X}}
\newcommand{\Y}{\mathcal{Y}}
\newcommand{\true}{\mathrm{true}}
\newcommand{\gtrue}{g_\mathrm{true}}
\newcommand{\ftrue}{f_\mathrm{true}}

\newcommand{\eiuu}{\mathrm{qNEIUU}}

\newcommand{\D}{\mathcal{D}}

\newcommand{\tf}{\tilde{f}}

\newcommand{\AF}{\mathrm{EUBO}}
\newcommand{\PE}{PE}

\DeclareMathOperator*{\argmax}{argmax}

\begin{document}

\runningtitle{Preference Exploration for Efficient Bayesian Optimization with Multiple Outcomes}

\twocolumn[

\aistatstitle{Preference Exploration for Efficient \\ Bayesian Optimization with Multiple Outcomes}

\aistatsauthor{Zhiyuan Jerry Lin \And Raul Astudillo \And  Peter I. Frazier \And Eytan Bakshy}

\aistatsaddress{Meta \And  Cornell University \And Cornell University \And Meta} ]

\begin{abstract}
We consider Bayesian optimization of expensive-to-evaluate experiments that generate vector-valued outcomes over which a decision-maker (DM) has preferences. These preferences are encoded by a utility function that is not known in closed form but can be estimated by asking the DM to express preferences over pairs of outcome vectors. To address this problem, we develop \emph{Bayesian optimization with preference exploration}, a novel framework that alternates between interactive \emph{real-time} preference learning with the DM via pairwise comparisons between outcomes, and Bayesian optimization with a learned compositional model of DM utility and outcomes. Within this framework, we propose preference exploration strategies specifically designed for this task, and demonstrate their performance via extensive simulation studies.
\end{abstract}

\newcommand\blfootnote[1]{%
  \begingroup
  \renewcommand\thefootnote{}\footnote{#1}%
  \addtocounter{footnote}{-1}%
  \endgroup
}

\section{INTRODUCTION}

Bayesian optimization (BO) is a sequential experimental design framework for efficient global optimization of black-box functions with expensive or time-consuming evaluations. It has succeeded in many real-world experimentation tasks, including materials design~\citep{frazier2016bayesian, packwood2017bayesian, zhang2020bayesian}, robot locomotion~\citep{calandra2016bayesian}, and internet experiments~\citep{letham2019constrained, mao2020real}.

This work focuses on a common practical problem faced by decision-makers (DMs) who wish to apply BO to time-consuming experiments with multiple outcomes of interest. DMs have unknown preferences over outcomes which can be elucidated via a limited set of interactions with the DM, and we wish to gather information through such interactions to support efficient experimentation.  Such problems commonly arise in A/B testing~\citep{bakshy2018ae} and simulation-based design~\citep{maddox2021optimizing}.

There are several possible approaches to BO with multiple outcomes in the literature, each with their own desiderata in our context. One approach is to have the DM express a fully-determined trade-off over outcomes via a function combining these outcomes into a single real-valued performance measure, and to perform single-objective BO with this function. Unfortunately, DMs are often unable to do this \citep{lepird2015bayesian}.

A second approach is multi-objective BO (MOBO) \citep{Hakanen2017-ie,feliot2018user,Abdolshah2019-qj}. MOBO aims to identify the entire feasible Pareto front but is typically inefficient because DMs are often interested only in a particular part of the Pareto front \citep{wang2017mini}.

A third approach directly presents a DM sets of $q$ designs  (most commonly, pairs, i.e., $q=2$) and asks the DM to express their preference over the sets. This data is used to model the DM's preferences over the designs directly~\citep{Brochu2008-cu,brochu2010interactive,Gonzalez2017-nx,siivola2020preferential}. We broadly refer to these methods as preferential BO (PBO).

PBO  methods can be implemented in our context by performing time-consuming experiments for the designs in the sets selected by a PBO algorithm, and then presenting their outcomes to the DM for comparison. However, such an approach would be inefficient in terms of time, experiment resources, and DM attention since DMs must wait for experiment results to complete to further input their preferences, which can be disruptive and time-consuming.  While higher throughput is possible if many designs are evaluated simultaneously~\citep{siivola2020preferential}, the rate at which additional information can be gathered about the DM's preferences is limited by the time-consuming experiments. Instead, we might learn more with less DM time via queries generated in real time using existing or hypothetical outcomes based on previously observed experimental data.

Toward overcoming the drawbacks of the above discussed approaches, we propose a novel human-in-the-loop algorithmic framework for optimizing multiple outcomes called \emph{Bayesian optimization with preference exploration (BOPE)}. In this framework, outcomes arise from a time-consuming-to-evaluate function $\ftrue : \R^d \rightarrow \R^k$, and the DM's preferences can be viewed as originating from a utility function $\gtrue : \R^k \rightarrow \R$, which is unknown but can be learned through the DM's responses to \emph{queries} in the form of comparisons between outcomes\footnote{For simplicity we consider pairwise comparisons, but generalizing BOPE to other comparisons is straightforward.}. The goal is to solve
\begin{equation}
\label{eq:main}
  \max_{x \in \X} \gtrue(\ftrue(x)),
\end{equation}
where $\X\subset\R^d$ is the \emph{design space}, using a limited number of queries to the DM and experiments (i.e., evaluations of $\ftrue$).

To do so, our framework iterates between two stages: \emph{preference exploration} (\PE) and \emph{experimentation}.
During a \PE{} stage, an algorithm (a \emph{\PE{} strategy}) generates a query consisting of two outcome vectors (i.e., elements of $\R^k$) for the DM to compare. The DM states their preferred outcome, and another query is presented in real time based on the result. Outcome vectors in these queries need not be generated by evaluating $\ftrue$. However, as we will see later, leveraging the available knowledge about $\ftrue$ can significantly improve performance. 
During an experimentation stage, an {\it experimentation strategy}
chooses a batch of points in the design space at which $\ftrue$ is evaluated.
In a \PE{} stage, previous DM queries and past experiment evaluations are used to choose the outcome vectors about which DM preference is elicited.
Similarly, in a experimentation stage, all information gathered by this point is used to determine design points that will be evaluated.

This approach provides four benefits.
First, relative to PBO, it supports models that decompose the latent objective function into separate models of the outcomes and the DM's preferences over these outcomes. This can improve prediction relative to PBO.
Second, it supports greater flexibility than PBO  when selecting user queries.
Queries can be selected adaptively after each batch of experiments and can include hypothetical outcomes generated via a predictive model of $\ftrue$, or outcomes of designs evaluated at previous points in time.  
Third, relative to MOBO, learned preferences allow focusing experimental attention on the relevant portion of the Pareto front, reducing the number of experiments needed.
Finally, in comparison with MOBO, our approach automatically handles non-monotone preferences over outcomes. This might arise, e.g., when designing a material that should neither be too stiff nor too flexible, or in drug discovery when a chemical's concentration in the blood should fall near a target.

\begin{figure}
    \centering
    \hspace{-0.1in}
    \includegraphics[width=0.95\linewidth]{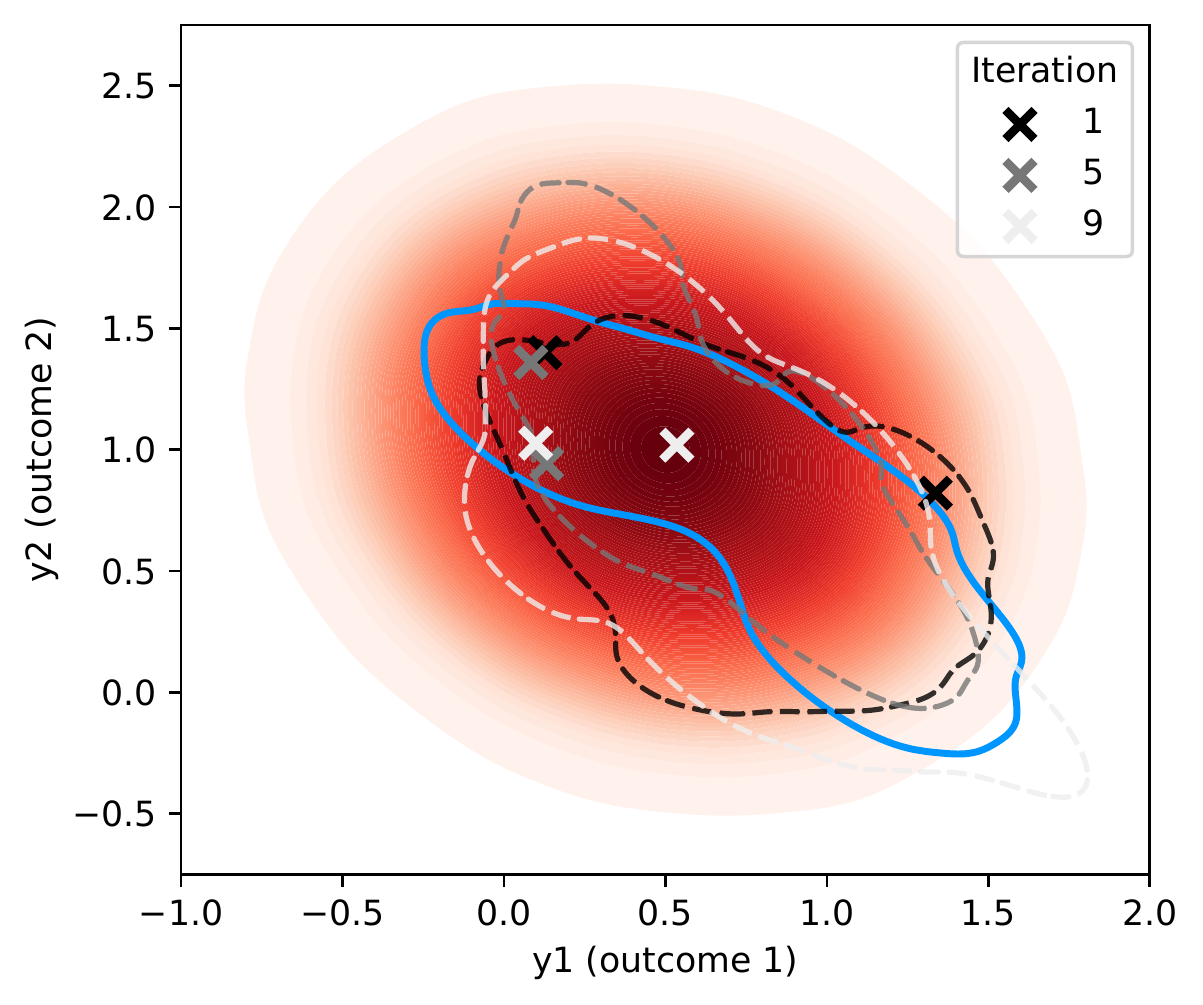}
    \caption{
    Preference exploration for a problem with a two-dimensional outcome space. The heatmap illustrates the DM's utility function $\gtrue(y)$ and the blue line circumscribes the space of possible outcomes achievable by $\ftrue(x)$ for any design $x \in \X$. Evaluation of $\ftrue(\cdot)$ is time-consuming. Our framework aims to collect preference data to support Bayesian optimization of $\gtrue(\ftrue(\cdot))$ with the aid of probabilistic surrogate models of outcomes ($f$) and DM utilities ($g$).
    $\times$s show outcome vectors presented to the DM by EUBO-$\tf$, over multiple iterations. 
    Each iteration queries regions likely to be of highest utility to the DM according to $g$ within search sets defined by independent sample paths  $\tf$ from $f$ (dashed loops). 
    This procedure helps learn a $g$ that may be used to select high-utility experiments.}

    \label{fig:example_eubo}
\end{figure}

This new workflow includes a key challenge unaddressed by prior work: how should preference information be gathered to best support optimization in such real-world contexts? Responding to this question, we examine several \PE{} strategies and show that the most successful ones leverage information about the posterior distribution of outcomes achievable under $\ftrue$.  Finally, we develop a one-step Bayes optimal \PE{} strategy for collecting preference information to solve \eqref{eq:main}. 

Figure~\ref{fig:example_eubo} illustrates one of our proposed
\PE{} strategies, EUBO-$\tf$.
 Rather than performing preference learning across all possible outcome vectors, which would take many queries, EUBO-$\tf$ 
focuses on queries comparing outcomes 
that are likely to be achievable.
This guides experiments toward regions of high utility to the DM in a small number of queries.

\paragraph{Contributions}
Our contributions are as follows:
\begin{itemize}[leftmargin=1.5em]
    \item We propose Bayesian optimization with preference exploration (BOPE), a novel human-in-the-loop framework for BO of time-consuming experiments that generate vector-valued outcomes over which a DM has unknown preferences. This framework reduces experimental cost and DM time over existing MOBO and PBO approaches.

    \item We develop a principled \PE{} strategy for adaptively selecting queries to present to the DM. The DM's answers to these queries localize the parts of the achievable region of outcomes that are of importance to the DM, and enable selection of high-utility designs to evaluate via experimentation. 
    \item We evaluate our approach on synthetic and real-world simulation problems, such as multi-objective vehicle design, demonstrating that \PE-based approaches significantly outperform MOBO, PBO, and other natural baselines.
\end{itemize}

\section{OTHER RELATED WORK}

A related stream of research has developed techniques for efficient elicitation of the DM's preferences over outcomes that are known {\it a priori} without the need for expensive evaluation
\citep{chajewska2000making, boutilier2002pomdp, Furnkranz2003-mu, Chu2005-ec, viappiani2010optimal,  Dewancker2016-ix}. 
While we leverage ideas from these works, they are not directly applicable in our setting, where outcomes are unknown and expensive-to-evaluate. 
Additionally, most of this literature assumes that there are a finite number of outcome vectors, while our space of potential outcomes is uncountably infinite.

Our work builds on \citet{Astudillo2020-jg} and \citet{lin2020_icml_pref}. The former considers sequential BO of vector-valued functions on behalf of a DM with unknown preferences using the same decomposition of outcomes and utilities considered here and proposes a strategy for the experimentation stage only.
The present paper builds off of a previous workshop paper by \citet{lin2020_icml_pref} which considers a similar problem setup as \citet{Astudillo2020-jg} but focuses on  batch-based PE and experimentation to support the types of workflows found in industry A/B testing settings. The present paper formalizes this setup, and provides a principled solution to the problem.

Finally, due to the composite structure of problem \eqref{eq:main}, our work is also related to BO of composite objective functions~\citep{Astudillo2019-ar, astudillo2021bayesian, astudillo2021thinking}, and we leverage similar computational techniques.

\section{BO WITH PREFERENCE EXPLORATION}
Here, we describe our framework, including the workflow and core models, and introduce key considerations for the design of strategies for PE and experimentation.%

\subsection{Workflow}
In our framework, the goal is to find a solution to~\eqref{eq:main} by learning the functions $\gtrue$ and $\ftrue$ while alternating between two stages: \textit{preference exploration} and \textit{experimentation}.

A \PE{} stage is a short uninterrupted period of time in which the DM interactively expresses preferences over multiple pairs of outcome vectors that does not involve the collection of new values from $\ftrue$.
We refer to the pair of outcome vectors presented to the DM as a {\it query} and to the DM's answer as the {\it response}.

An experimentation stage is a period of time in which one or more evaluations of the outcome function $\ftrue$ are performed, typically in parallel. Our presentation assumes that \PE{} and experimentation stages alternate. As a special case, this can include situations where \PE{} is performed only once after an initial round of experimentation, so as to take the human out of the loop for subsequent experiments. For example, in the context of internet experimentation, an experimentation stage would entail running a batch of A/B tests in parallel and a \PE{} stage could be an interactive session with a data scientist where pairs of (potentially hypothetical) experimental outcomes are compared.

Once all PE and experimentation stages are complete, the DM is shown the outcome vectors for all experiments that have been performed, and the DM selects their preferred design.\footnote{In practice, the learned utility may be used to rank these items in advance to reduce the cognitive burden of sorting through all evaluated designs.}

\subsection{Outcome and Preference Models}
\label{sec:models}
We utilize two probabilistic surrogate models, one for the outcome function, $f_\true$, and one for the utility function, $g_\true$.

To model the outcome function, $\ftrue$, we use a multi-output Gaussian process (GP), $f$, characterized by a prior mean function $\mu_0^f:\X\rightarrow\R^k$, and a prior covariance function, $K_0^f:\X\times\X\rightarrow\R^{k\times k}$. Given a dataset of $n$ (potentially noisy) observations of the outcome function, $\D_n = \{(x_i, y_i)\}_{i=1}^n$, the probabilistic surrogate model of $\ftrue$ is then given by the posterior distribution of $f$ given $\D_n$, which is again a multi-output GP with mean and covariance functions $\mu_n^f:\X\rightarrow\R^k$ 
and $K_n^f:\X\times\X\rightarrow\R^{k\times k}$ that can be computed in closed form in terms of $\mu_0^f$ and $K_0^f$.

The utility function, $\gtrue$, is also modeled using a GP, $g$, which again requires specifying a prior mean function, $\mu_0^g:\R^k\rightarrow\R$, and a prior covariance function, $K_0^g:\R^k\times\R^k\rightarrow\R$.

Given a query $(y_1,y_2)$ constituted by two outcome vectors, we let $r(y_1, y_2) \in \{1,2\}$ indicate whether the DM preferred the first or second outcome vector offered. Following \cite{Chu2005-ec}, we assume that the DM's responses are distributed according to a probit likelihood of the form
$$\prob(r(y_1, y_2) = 1 \mid g(y_1), g(y_2)) = \Phi\left(\frac{g(y_1)-g(y_2)}{\sqrt{2}\lambda}\right),$$ where $\lambda$ is a hyperparameter that can be estimated along with the other hyperparameters of the model, and $\Phi$ is the standard normal CDF.

In our experiments, we use the Laplace approximation suggested by \cite{Chu2005-ec}, which results in an approximate posterior of $g$ that is again a GP.
When we have observed the results of $m$ queries, $\mathcal{P}_m = \left\{\left(y_{1,j}, y_{2,j}, r(y_{1,j}, y_{2,j})\right)\right\}_{j=1}^m$, we let $\mu_m^g$ and $K_m^g$ refer to the mean and covariance functions of this approximate GP posterior.

\subsection{Preference Exploration Strategies}
\label{sec:PE-strategy}

Preference exploration strategies aim to select queries $(y_1,y_2) \in \Y \times \Y$, where $\Y \subseteq \R^k$, so as to best support experiment selection. Here, we introduce three classes of PE strategies investigated in this work.

\paragraph{PE Strategies That Learn Preferences Over a Prior Region of Interest}
Our first class of PE strategies requires choosing a prior set $\Y$ likely to contain most or all of the achievable region $\ftrue(\X) = \{\ftrue(x) : x \in \X\}$, so that preferences over $\Y$ are highly informative of preferences over $\ftrue(\X)$.
It then chooses queries to learn DM preferences over $\Y$. In the simplest case, $\Y$ can be 
a hyperrectangle bounding a likely minimum and maximum for each outcome provided by the DM. Alternatively, $\Y$ could be estimated via a meta-analysis of related experiments, and could also incorporate information about how outcomes covary across experiments.

There is a tradeoff in choosing the bounds of $\Y$. Choosing $\Y$ to be too small risks excluding relevant potential outcomes. Choosing $\Y$ to be much larger than $\ftrue(\X)$ can cause over-exploration of areas not relevant to the optimization task. This may occur with hyper-rectangular $\Y$ if correlated outcomes make $\ftrue(\X)$ much smaller than $\Y$.

Two policies for generating queries to learn preferences over $\Y$ are: selecting queries uniformly at random over $\Y\times\Y$; and selecting queries by maximizing an active learning acquisition function (AF) such as Bayesian active learning by disagreement (BALD)~\citep{Houlsby2011-jf} over $\Y\times\Y$.

\paragraph{PE Strategies That Leverage Direct Experimental Data} Another class of PE strategies samples many plausible achievable regions based on data from $\ftrue$. For each such region $\Y$, it seeks to learn the DM's preferences over $\Y$.
Learning these preferences allows us to eventually learn the best point in the true achievable region, despite not knowing this region.

Probabilistic surrogate models provide a natural mechanism for sampling plausible achievable regions. In many PE strategies considered here, for each preference query, a GP sample path $\tf$ is first drawn from $f$, implying an associated achievable region $\Y = \tf(\X)$.
Then a query is designed to improve our knowledge of the DM's most preferred design in $\Y$.
Compared to the former class of PE strategies that learn preferences over a prior region of interest,
this approach aims to reduce the number of queries needed by learning a separate ranking over many smaller sets $\Y$. As we learn more about the achievable region through experimentation, the sequence of sampled sets $\Y$ concentrates, further reducing the size of the query space. %

We consider three strategies for learning preferences over $\Y = \tf(\X)$: random search, BALD, and a novel AF called EUBO which is introduced in \S\ref{sec:known} and aims to find the best query over a known set of outcomes. \S\ref{sec:eubo_zeta} also derives an approximation of a one-step optimal AF for BOPE that has a similar structure but uses a different choice of $\Y$.

\paragraph{Optimization of PE AFs} Several of the PE strategies discussed above require optimizing an AF $\alpha$ over $\Y\times\Y$. In several cases, $\Y$ can be written as $\Y = \{h(x) : x \in \X\}$ for some deterministic function $h$. This makes optimization convenient, since is allows for optimization over a known domain, $\X$: $\argmax_{y_1, y_2 \in \Y} \alpha(y_1, y_2) = \argmax_{x_1, x_2 \in \X} \alpha(h(x_1), h(x_2))$.
Thus, to find the maximizer of $\alpha(y_1,y_2)$ over $y_1, y_2 \in \Y$, it is sufficient to find the maximizer of $\alpha(h(x_1), h(x_2))$ over $x_1, x_2 \in \X$.

\subsection{Experiment Selection Strategies}
\label{sec:experimentation}
Here, we discuss the strategy we use to select the designs at which $\ftrue$ is evaluated during the experimentation stages. Since the focus of our work is on \PE{} strategies, we restrict our attention to a single experiment selection strategy. We propose a generalization of the expected improvement under utility uncertainty (EIUU) AF, introduced by \citet{Astudillo2020-jg}, to support noisy and parallel evaluations with a non-parametric utility model $g$. This AF integrates over the uncertainty of both $f$ and $g$ when selecting the design points.

Formally, for a batch of $q$ points $x_{1:q} = (x_1, \ldots, x_q) \in \X^q$, this AF is defined by
\begin{align*}
 \eiuu&(x_{1:q}) = \\
  &\E_{m,n}\left[\left\{\max g(f(x_{1:q})) -  \max g(f(X_n))\right\}^{+}\right],
\end{align*}
where $\E_{m,n}[\cdot] = \E[\cdot \mid \mathcal{P}_m, \D_n]$ denotes the conditional expectation given the data from $m$ queries and $n$ experiments,  $\{\cdot\}^+$ denotes the positive part function, and, making a slight abuse of notation, we define
$ \max g(f(x_{1:q})) = \max_{i=1,\ldots,q}g(f(x_i))$
and
$\max g(f(X_n)) = \max_{(x,y)\in\D_n}g(f(x)).$

$\eiuu$ can be straightforwardly implemented as a Monte Carlo AF by applying the reparametrization trick to both $f$ and $g$, and optimized via sample average approximation (SAA)~\citep{wilson2018maxbo,balandat2020botorch}. We refer the reader to \S\ref{sec:qneiuu_implementation} in the supplementary material (SM) for more details on the implementation of $\eiuu$.

\section{ONE-STEP OPTIMAL PREFERENCE EXPLORATION}
\vspace{-0.25em}
Here we present a principled \PE{} strategy derived using a one-step Bayes optimality analysis.
We begin by describing a one-step optimal strategy for learning preferences over a  known set of outcome vectors as well as results regarding its efficient computation.  We then propose a one-step optimal \PE{} strategy that formally accounts for uncertainty over the set of achievable outcomes.  This strategy is not practical for real-time learning, but we provide a faster principled approximation using insights developed in the case with known achievable outcomes.

\subsection{Preference Exploration Over a Known Set of Outcomes}
\label{sec:known}

Here we assume that the space of achievable outcomes is known and denote it by $\Y$.

To motivate our AF, we ask the following rhetorical question: If we had to offer a single outcome vector $y^*\in \Y$ to the DM using the available information so far, what would the right choice be? We argue that a sensible choice is $y^*$ so that the expected utility received by the DM is maximal; i.e., $y^*\in \argmax_{y\in \Y}\E_m[g(y)]$, where $\E_m$ denotes the conditional expectation given $\mathcal{P}_m$ (i.e., $\E_m[\cdot] = \E[\cdot\mid \mathcal{P}_m]$). Following this logic, if we were allowed to ask an additional query $(y_1, y_2)$ and observe the DM's response $r(y_1,y_2)$, $y^*$ would now be chosen so that $y^*\in \argmax_{y\in \Y}\E_{m+1}[g(y)]$, where  $\E_{m+1}[\cdot] = \E\left[\cdot\mid \mathcal{P}_m \cup \left\{\left(y_1, y_2, r(y_1,y_2)\right)\right\}\right]$. 
Thus, $$\max_{y\in \Y}\E_{m+1}[g(y)] - \max_{y\in \Y}\E_m[g(y)]$$ quantifies the difference in (expected) utility obtained by the DM due to the additional query. Our AF can now be defined as the expectation of this difference given the information available so far; i.e., 
\begin{align*}
V(y_1, y_2) &= \E_m\left[\max_{y\in \Y}\E_{m+1}[g(y)] - \max_{y\in \Y}\E_m[g(y)]\right],
\end{align*}
where the dependence of the right-hand-side on $(y_1, y_2)$ is made implicit by our notation $\E_{m+1}[\cdot] = \E\left[\cdot\mid \mathcal{P}_m \cup \left\{\left(y_1, y_2, r(y_1,y_2)\right)\right\}\right]$. We also note that the term $\max_{y\in \Y}\E_m[g(y)]$ does not depend on $(y_1, y_2)$ and thus can be disregarded when maximizing $V$.
This AF can be considered to be in the knowledge gradient family of AFs \citep{Frazier2018-qn}, because it values information according to its impact on the maximum posterior expected utility.

\subsubsection{Expected Utility of the Best Outcome}
The acquisition function $V$ defined above is challenging to maximize directly due to its nested structure, as is typical of knowledge gradient AFs.
Fortunately, the following theorem shows that maximizing $V$ is equivalent to maximizing another AF that is easier to optimize.

We define the \textit{expected utility of the best option\,(EUBO)},
\begin{equation*}
\AF(y_1, y_2) = \mathbb{E}_m\left[
\max\{g(y_1), g(y_2)\} 
\right],
\end{equation*}
where the expectation is over the posterior of the DM utility $g$ at the time the query is chosen. $V$ and $\AF$ are related via the following result.
\begin{theorem}
\label{thm1}
If $\lambda=0$, and the posterior mean $\mu^g_m$ and posterior covariance $K^g_m$ are both continuous, then
\begin{equation*}
\argmax_{y_1, y_2 \in \Y} \AF(y_1,y_2) \subseteq \argmax_{y_1, y_2 \in \mathbb{R}^k} V(y_1,y_2)
\end{equation*}
and the left-hand side is non-empty.
\end{theorem}

Thus, the above result shows that, when DM responses occur without error, one can find a maximizer of $V$ by maximizing $\AF$ instead, which, as we argue below, is a significantly simpler task. While this holds for $\lambda = 0$ only, the following result shows that maximizing $\AF$ yields a high-quality solution even if $\lambda > 0$.

\begin{theorem}
\label{thm2}
Denote $V$ as $V_\lambda$ to make the dependence on $\lambda$ explicit, and let $(y_1^*, y_2^*) \in \argmax_{y_1, y_2}\AF(y_1, y_2)$. Then,
\begin{equation*}
    V_\lambda(y_1^*, y_2^*) \geq \max_{y_1, y_2 \in \Y}V_0(y_1, y_2) -\lambda C,
\end{equation*}
where $C=e^{-1/2}/\sqrt{2}$.
\end{theorem}

Theorem~\ref{thm2} provides a lower bound on the acquisition value of a maximizer of EUBO evaluated at $V$ in the presence of comparison noise.
The proofs of Theorems 1 and 2 can be found in the SM \S\ref{sec:theoretical}.

Maximizing $\AF$ is easier than $V$ for two reasons. First, computing $\AF$ does not require solving the inner optimization problem $\max_{y\in \Y}\E_{m+1}[g(y)]$ required by $V$. Second, $\AF$ can be expressed in closed form under the approximate Gaussian posterior of $g$ implied by the Laplace approximation, as shown in SM \S\ref{sec:analytic_eubo}.%

\subsection{Preference Exploration Under an Unknown Set of Achievable Outcomes}
\label{sec:pe_kg}
We now describe a \PE{} strategy that formally takes into account uncertainty on the set of achievable outcomes. To support our analysis, here we assume that evaluations of the outcome function are noise-free. 

As in the previous subsection, we derive a PE strategy using a one-step optimality analysis. Formally, this strategy selects the query $(y_1,y_2)$ that is optimal with respect to the following sequence of actions (which constitute one step):
\begin{enumerate}
    \item Select query $(y_1, y_2)\in\R^k\times\R^k$ and observe DM's response $r(y_1, y_2)$,
    \item Select design $ x\in\X$ and observe outcome $f(x)$,
    \item Obtain reward $\max_{i=1,\ldots, n+1}\E_{m+1,n+1}[g(f(x_i))]$, where we define $x_{n+1} = x$ and $\E_{m+1,n+1}[\cdot]$ denotes $\E\left[\cdot\mid \mathcal{P}_m \cup \left\{\left(y_1, y_2, r(y_1,y_2)\right)\right\}, \D_n \cup \left\{(x, f(x))\right\}\right]$.
\end{enumerate}
The optimal query can be found by solving $\max_{y_1,y_2\in\R^k} W(y_1,y_2)$, where
\begin{align*}
W&(y_1,y_2)  \\
    = & \E_{m,n}[\max_{x_{n+1}\in\X}\E_{m+1,n}[\max_{i=1,\ldots, n+1}\E_{m+1, n+1}[g(f(x_i))]]]
\end{align*}
and $ \E_{m+1, n}[\cdot] = \E\left[\cdot\mid \mathcal{P}_m \cup \left\{\left(y_1, y_2, r(y_1,y_2)\right)\right\}, \D_n\right]$

This strategy is similar in spirit to the one described in the previous subsection. However, there are two key differences, which originate due to the set of achievable outcomes being unknown. First, there is an additional action  between query selection and reward collection. Here, one additional evaluation of the outcome function is performed, thus enforcing the need to gather preference information to support experimentation (i.e., evaluations of the outcome function). Second, since the set of achievable outcomes is unknown, the reward is computed over the outcomes observed so far only.

\subsection{Efficient Approximate Maximization of $W$ Via a Single-Sample Approximation}
\label{sec:eubo_zeta}
Unsurprisingly, $W$ is quite hard to compute and optimize due to its nested structure. In principle, one could aim to adapt optimization strategies for lookahead AFs (see, e.g., \citealt{balandat2020botorch,jiang2020efficient}). However, these methods are quite computationally expensive, with run times on the order of several minutes per acquisition, making them impractical in the context of real-time \PE. Instead, we derive an efficient approximate optimization scheme based on a single-sample approximation of $W$ with respect to the uncertainty on $f(x_{n+1})$.

Leveraging the approximate GP distribution over $g$ induced by the Laplace approximation, we write $\E_{m+1, n+1}[g(f(x_i))]$ as $\mu_{m+1}^g(f(x_i))$ for $i=1,\ldots, n+1$,\footnote{We leverage this closed form expression to simplify our notation but this is not critical. Our analysis holds even if $\E_{m+1, n+1}[g(f(x_i))]$ does not have a closed form expression.} and define $\mu_{m+1,n}^* = \max_{i=1,\ldots, n}\mu_{m+1}^g(f(x_i))$. Applying the reparametrization trick on $f(x_{n+1})$, $W$ can be rewritten as
\begin{align*}
 W&(y_1,y_2) = \\
   & \E_{m,n}[\max_{x\in\X}\E_{m+1,n}[\max\{\mu_{m+1,n}^*, \mu_{m+1}^g(\zeta_n(x; Z))\}]],
\end{align*}
where  $\zeta_n(x; Z) = \mu_n^f(x) + C_n^f(x)Z$, $C^f_n(x)$ is the lower Cholesky factor of $K_n^f(x,x)$, and the (conditional) distribution of $Z$ is the standard normal. In the expression above, the inner expectation is over $Z$ and the outer expectation is over $r(y_1, y_2)$.

If we approximate the expression above using a single sample from $Z$, which we denote by $\tilde{Z}$, and making a slight abuse of notation, we obtain the approximation $W(y_1,y_2) \approx  W(y_1,y_2; \tilde{Z})$, where
\begin{align*}
W&(y_1,y_2; \tilde{Z}) =\\
     & \E_{m,n}[\max_{x\in\X}\max\{\mu_{m+1,n}^*, \mu_{m+1}^g(\zeta_n(x;\tilde{Z}))\}],    
\end{align*}
and $\tilde{Z}$ is deterministic in the expectation above. Moreover, since $\zeta_n(x_i;\tilde{Z}) = f(x_i)$ for previously evaluated points $i=1,\ldots, n$,
$\max_{x\in\X}\mu_{m+1}^g(\zeta_n(x;\tilde{Z})) \geq \mu_{m+1,n}^*$ and thus  $\max_{x\in\X}\max\{\mu_{m+1,n}^*, \mu_{m+1}^g(\zeta_n(x;\tilde{Z}))\}$ can be simplified to $\max_{x\in\X}\mu_{m+1}^g(\zeta_n(x;\tilde{Z}))$. Thus, 
\begin{align*}
 W(y_1,y_2; \tilde{Z}) = \E_{m,n}[\max_{x\in\X}\mu_{m+1}^g(\zeta_n(x;\tilde{Z}))].   
\end{align*}

We can now use the machinery developed in \S\ref{sec:known} to efficiently maximize $W(y_1,y_2; \tilde{Z})$. Concretely, if we let $\Y = \{\zeta_n(x;\tilde{Z}) : x\in\X\}$, it follows from Theorem~\ref{thm1} that $\argmax_{y_1, y_2 \in \Y}\AF(y_1, y_2) \subseteq \argmax_{y_1, y_2 \in \Y}W(y_1,y_2; \tilde{Z})$ when $\lambda = 0$. Analogously, Theorem~\ref{thm2} provides a guarantee on the quality of a query in $\argmax_{y_1, y_2 \in \Y}\AF(y_1, y_2)$ when $\lambda > 0$.  We call the resulting \PE{} strategy EUBO-$\zeta$.

While EUBO-$\zeta$ is derived as an approximation of a one-step optimal strategy, it has a structure similar to that of the strategies discussed in \S\ref{sec:PE-strategy}: it builds a set $\Y$ using $f$ and then selects a query in $\Y\times\Y$ by maximizing an AF. This perspective leads us to consider the PE strategy that chooses queries by maximizing $\AF(y_1,y_2)$ over $y_1,y_2 \in \tf(\X)$. We call the resulting strategy EUBO-$\tf$. This variation of EUBO provides a more intuitive interpretation than EUBO-$\zeta$.

\vspace{-0.25em}
\section{NUMERICAL EXPERIMENTS}
\vspace{-0.5em}
\label{sec:sim_study}

\begin{figure*}
    \centering
    \includegraphics[width=1\textwidth]{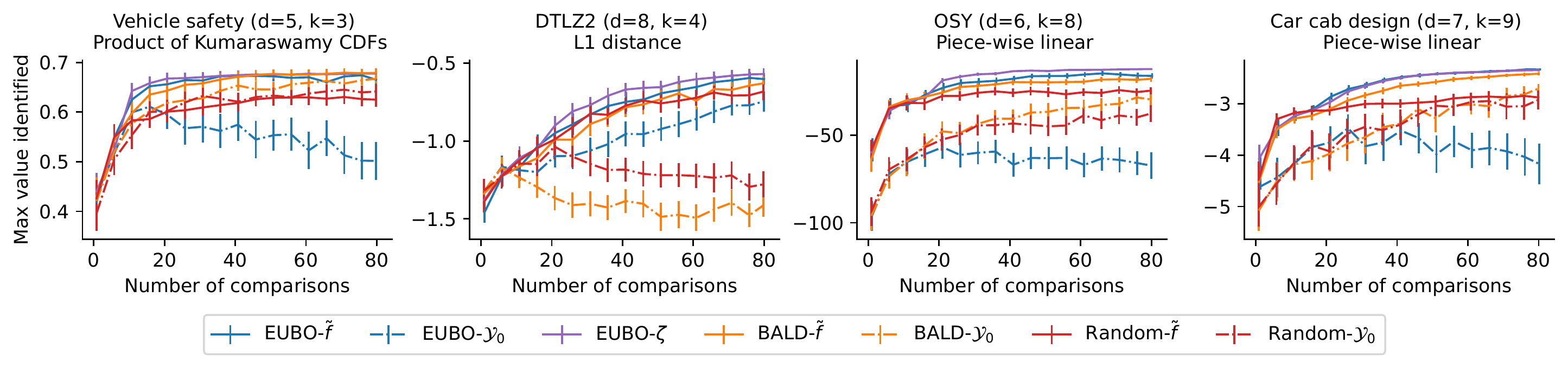}
    \vspace{-0.1in}
    \caption{
    Mean utility of designs chosen according to maximum posterior predictive mean after a given number of pairwise comparisons during the first stage of preference exploration. CIs are $\pm$2 standard errors of the mean across 100 simulation replications.
    }
    \label{fig:within_batch_sim}
\end{figure*}

\begin{figure*}[h!]
    \centering
    \includegraphics[width=1\textwidth]{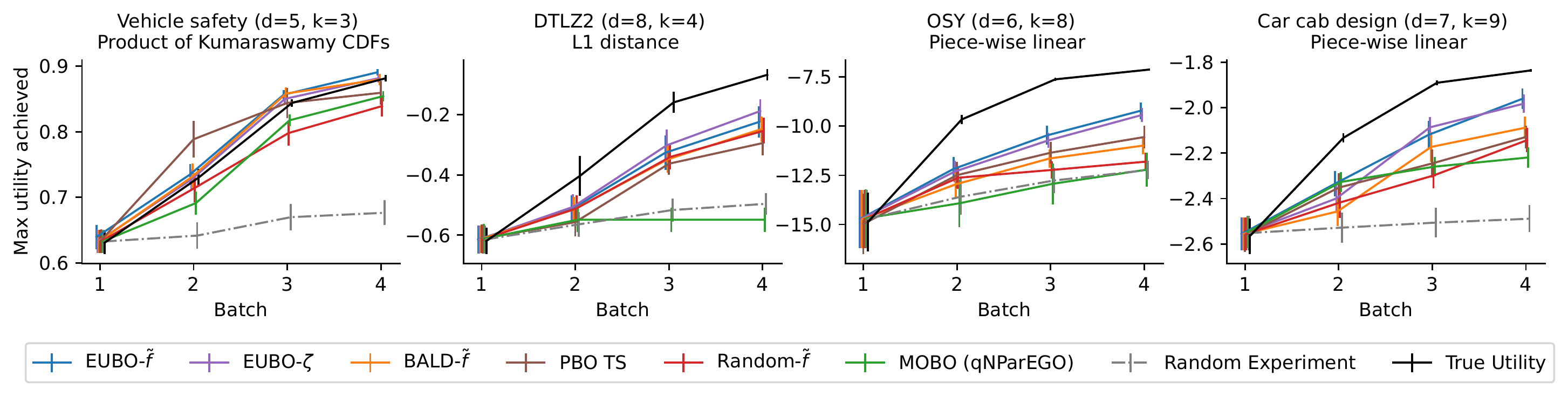}
    \vspace{-0.2in}
    \caption{
    Max utility achieved by interleaving batches of preference exploration and experimentation.
    ``True Utility'' shown in black represents an approximate upper bound on the performance achievable obtained via Bayesian optimization with known utility. CIs are $\pm$2 standard errors of the mean across 30 simulation replications.
    }
    \label{fig:multi_batch_sim} 
    \vspace{-0.1in}
\end{figure*}

We evaluate our proposed strategies on real-world and synthetic outcome functions as well as several utility functions.
The main text considers four test problems: vehicle safety ($d=5, k=3)$~\citep{liao2008multiobjective, tanabe2020easy},
DTLZ2~\citep{deb2005scalable} ($d=4, k=8)$, OSY ($d=6, k=8)$~\citep{deb2013evolutionary}, and car cab design ($d=7, k=9)$ problems~\citep{tanabe2020easy}.
These test problems are matched with several utility functions: linear and piece-wise linear utilities, the product of Kumaraswamy distribution CDFs, modeling soft constraints,
and the L1 distance from a Pareto-optimal point.
All test problems are described in detail in \S\ref{sec:test_functions} of the SM. Results are qualitatively similar to those presented here. Table~\ref{tab:complete_test_func} in the SM
summarizes all test outcome and utility functions considered in this work.

With these outcome and utility functions, we perform three types of simulation-based experiments.
\S\ref{sec:within_batch_sim} considers a single \PE{} stage. With data from a single batch of experiments, we train a surrogate outcome model $f$, then use \PE{} strategies to identify high utility designs in a single \PE{} stage. 
\S\ref{sec:mulit_batch_sim} considers
the case in which multiple PE and experimentation stages are interleaved, and
evaluates the maximum utility found over several rounds of BOPE for each \PE{} strategy.
Finally, \S\ref{sec:oneshot} considers a setting in which all preference exploration occurs in a single stage, and this is followed by multiple batches of experimentation without further intervention from the DM.
All simulations compare EUBO-based \PE{} strategies against several other \PE{} strategies.
Unless otherwise noted, we use the $\eiuu$ experiment selection strategy.

To emulate noise in preferences expressed by human DMs,
simulated DMs in all experiments select the option with lower utility 10\% of the time. In \S\ref{sec:probit_sims} of the SM, we also experiment with a different probit comparison error and observe similar results.

Complete simulation results with further baselines, test problems, and settings are available in  \S\ref{sec:additional_sims} of the SM.

\paragraph{Acquisition strategies for PE.} We examine several PE strategies described in \S\ref{sec:PE-strategy}.
This includes EUBO-$\zeta$ and methods selecting queries from one of two sets $\Y$: 
(i) a hyper-rectangle $\Y_0$ bounding $\ftrue(\X)$, and (ii)  $\tf(\X)$.
$\Y_0$ is constructed to be the smallest hyper-rectangle that contains $\ftrue(\X)$, estimated via $10^8$ Monte Carlo samples from $\X$. Thus, $\Y_0$ provides optimistic baseline of a DM that can perfectly characterize the upper and lower (box) bounds of $\ftrue$.
$\tf$ is sampled from the posterior on $f$ via random Fourier features with 512 basis functions~\citep{rahimi2007random}.

To select queries from these two sets $\Y$, we consider random search, BALD, and EUBO. We name these algorithms via their sampling strategy and choice of $\Y$: Random-$\Y_0$ and Random-$\tf$, and similarly for EUBO and BALD. BALD-$\Y$ (for both choices of $\Y$) is estimated via quasi-Monte Carlo (QMC), and  selects designs $x \in \X$ to reduce the posterior uncertainty of $g$ over $\Y$. Variants of EUBO are computed using the closed form expression derived in \S\ref{sec:analytic_eubo}  of the SM.

All algorithms are implemented in BoTorch~\citep{balandat2020botorch}. All AFs are optimized via SAA using L-BFGS-B. We use a Mat\'{e}rn 5/2 ARD covariance function for the outcome model and RBF ARD kernel for the preference model. We refer the reader to \S\ref{sec:sim_design_detail} of the SM for additional implementation details.

\subsection{Identifying High Utility Designs with PE}
\label{sec:within_batch_sim}
We first examine how the proposed \PE{} strategies identify design points whose outcomes have high utilities during a single \PE{} stage.
We first evaluate the outcome function at a batch of quasi-random design points and fit a multi-output GP to the observed data.
Outcome functions with five or fewer input dimensions receive 16 initial designs and the remainder receive 32. We then initialize the preference model using pairs of random designs  from the initial batch for the first 2$k$ pairwise comparisons, followed by comparisons acquired via  PE strategies. The outcome surrogate model remains unchanged.

We plot the true utility of our best guess at the utility-optimal design after every 5 pairwise comparisons. To do so, we maximize $\mathbb{E}[g(f(x))]$ over $x\in \X$ where the expectation is taken under the posterior given all available query responses and experiment results.
We maximize following a SAA approach, sampling over realizations of $g(f(x))$.  This provides the design $\hat{x}$ with the maximum posterior mean. We plot $\gtrue(\ftrue(\hat{x}))$ as the utility earned by a given \PE{} method after a given number of pairwise comparisons.

Figure~\ref{fig:within_batch_sim} shows results. EUBO-$\zeta$ and EUBO-$\tf$ perform at least as well as baseline strategies across all test problems on anytime performance, while BALD-$\tf$ achieves competitive but slightly inferior performance. EUBO-$\Y_0$ tends to over-explore outcomes that are not achievable under $\ftrue(\X)$, leading to models of $g$ that are progressively less accurate estimates of the posterior maximizer of $\gtrue(\ftrue(\cdot))$.
\vspace{-0.25em}
\subsection{BOPE with Multiple PE Stages}
\vspace{-0.25em}

\label{sec:mulit_batch_sim}

\begin{figure*}[h]
    \centering
    \includegraphics[width=1\textwidth]{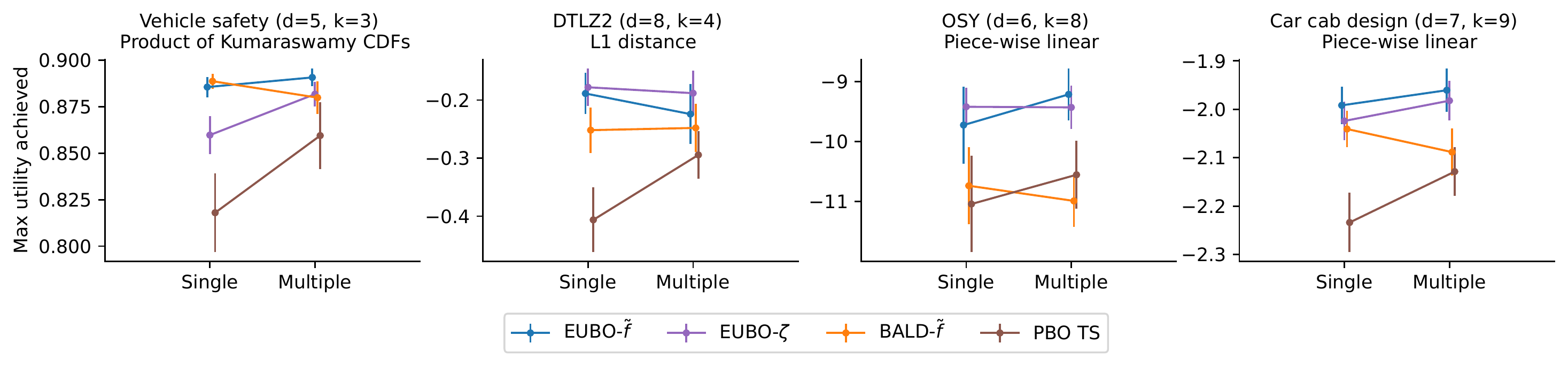}
    \vspace{-0.25in}
    \caption{Max utility achieved a single round or multiple rounds of PE after the last experimentation batch. Outcomes are only marginally improved through the use of multiple PE stages for EUBO-based acquisition functions, but PBO-based strategies benefit greatly from the ability to elicit preferences across multiple rounds of experimentation. The opposite effect is seen for active learning strategies, which benefit most from conducting all learning upfront. Plots show mean and $\pm$2 standard errors of the mean for 30 simulation replicates.}
    \label{fig:mt_interactive_oneshot_diff_in_util}
    \vspace{-0.15in}
\end{figure*}

We evaluate our proposed methods in a full BOPE loop with alternating rounds of experimentation and \PE. After each experimentation stage,  a \PE{} stage is performed using a \PE{} strategy. The learned preference model is then used to select designs for subsequent experiments. Each \PE{} stage elicits 25 pairwise comparisons from the DM. This occurs over 3 rounds of experimentation, leading to a total of 75 comparisons.

Following the previous subsection, the first experimentation stage uses 32 points generated with a Sobol sequence for higher-dimensional outcome functions ($d > 5$), and 16 otherwise.
For each subsequent batch, \PE{} is performed, a batch of 16 (or 8 for the vehicle safety problem) design points are generated for each subsequent experimentation stage (i.e., batches 2 - 4).
\PE{} strategies are used in combination with $\eiuu$ to select designs in batches 2-4. In addition to high-performing baselines examined in \S\ref{sec:within_batch_sim}, we include a few additional baselines.  

We adapt PBO to the BOPE setting as follows: For PE, we repeatedly apply PBO AFs to the results of previous experiments to elicit DM's responses over previously observed outcomes: $(y_i, y_j)$ with $y_i, y_j \in \{y:(x,y)\in\D_n)\}$. A standard GP model with a Laplace approximation is used to directly model the latent objective value; i.e., the mapping $x\mapsto \gtrue(\ftrue(x))$. We consider an adaptation of Thompson sampling, a popular AF used in the PBO literature~\citep{siivola2020preferential}. This strategy is performed by selecting $x \in \mathcal{D}_n$ with the  value based on independent samples from the posterior distribution of the latent objective. We refer to this strategy as PBO TS. Experiments are selected via a Monte Carlo implementation of the noisy expected improvement AF~\citep{letham2019constrained}, qNEI. We also considered a PBO AF based on EUBO, but the results were similar to those of PBO TS.

We additionally consider the multi-objective optimization algorithm qNParEGO~\citep{daulton2020differentiable}, single-objective BO with the ground truth utility function (True Utility), and a strategy where we experiment with random Sobol design points, which are approximate upper and lower bounds on a \PE{} strategy's performance. The True Utility is optimized with qNEI  using a compositional objective \citep{Astudillo2019-ar, balandat2020botorch}.

Figure~\ref{fig:multi_batch_sim} shows that in all test functions presented here, BO using EUBO-$\zeta$ and EUBO-$\tf$  consistently achieves the highest utility, only second to the ground truth utility. In our experiments, we find that PBO-TS
does not achieve the same level of performance as PE AFs. We also consider PE AFs that perform search directly in the entire $\Y_0$ domain in the SM.  Similar to the previous set of experiments, search directly in $\Y_0$ tends to perform worse than search in $\tf(\X)$. %

\vspace{-0.25em}
\subsection{BOPE with a Single PE Stage}
\label{sec:oneshot}
\vspace{-0.25em}

DM interruptions are costly in practice.
Therefore, we consider the case in which an initial experimentation stage takes place, followed by a single 75-comparison PE stage, after which experimentation proceeds in 3 batches of the same size as \S\ref{sec:mulit_batch_sim}.\footnote{The single \PE{} stage approach makes it feasible to use smaller batch sizes or fully sequential optimization, but we use identical batch sizes to allow for more direct comparison.}

While performing \PE{} only once minimizes DM interruptions, it may adversely impact the optimization since all learning occurs with a surrogate of a small amount of initial random design points.
Figure~\ref{fig:mt_interactive_oneshot_diff_in_util} compares the maximum utility achieved after performing four batches of experimentation, when performing all \PE{} after the first batch, vs performing \PE{} between rounds of experimentation.
For all PE-based methods, the maximum utilities achieved are not statistically different in both settings, suggesting they are robust to different levels of PE interactivity under our experimental setting.
On the other hand, we observe rather significant improvements in maximum utility achieved with multiple \PE{} stages for PBO-TS. In this case, we see that the best value achieved quickly plateaus after the second batch of experimentation   (Figure~\ref{fig:oneshot_multi_batch_sim} in SM). %

\section{CONCLUSION}
\vspace{-0.55em}
\label{sec:conclusion}
BO is a prominent method for 
sequential experimental design, often promising to ``take the human out of the loop''~\citep{shahriari2015taking}.
However, in practice, human DMs often struggle to describe the objective they wish to optimize.
This work proposes a novel human-in-the-loop BO framework with interleaved preference exploration stages called \emph{Bayesian optimization with preference exploration}, where humans and algorithms collaborate to learn DMs' preferences over plausible outcomes for a particular black-box optimization task. 

We propose EUBO, a simple and computationally efficient algorithm for exploring and learning the DM's preferences.
These learned preferences in turn enable efficient search for designs whose outputs have high utilities via BO. We show that EUBO-$\zeta$ is an approximate one-step optimal policy for learning a DM's preferences with respect to our current knowledge of both the DM's utility function and the outcome function, which allows us to focus preference exploration on the most relevant parts of the outcome space with respect to the current data.  
EUBO provides query efficiency improvements relative to benchmark methods, finding higher utility designs while consuming less DM time and requiring fewer experiments.

This work suggests areas for future research.  Real-time performance was enabled by approximate versions of the preference model and the optimal one-step optimal strategy $W$.  State-of-the-art models such as SkewGPs~\citep{benavoli2021preferential}, and more accurate approximations of $W$ could potentially improve PE sample complexity while still supporting real-time interaction.  BALD-$\tf$ is a strong and fast baseline, suggesting that new information-theoretic PE strategies could perform well in the BOPE framework.

\vspace{-0.25em}
\paragraph{Replication Material}
Code for replicating experiments in this paper is available at \url{https://github.com/facebookresearch/preference-exploration}

\subsubsection*{Acknowledgements}
{}PF and RA were supported by AFOSR FA9550-19-1-0283. We thank the anonymous reviewers as well as Sait Cakmak, Sam Daulton,  Daniel Jiang, Ben Letham, Yujia Zhang, and Yunxiang Zhang for their feedback on this work.  We thank Max Balandat for his review and guidance in the implementation of the preference model used in this work. We also thank Michael Shvartsman for his Monte Carlo-based implementation of BALD.

\bibliographystyle{plainnat}
\bibliography{ref}

\clearpage
\appendix

\thispagestyle{empty}

\onecolumn \makesupplementtitle

\section{DETAILS ON ACQUISITION FUNCTIONS}
\label{sec:sm_acqf_details}
\subsection{Implementation of qNEIUU}
\label{sec:qneiuu_implementation}
Here we describe our approach to compute and optimize the qNEIUU acquisition function. Succinctly, we follow the approach of ~\citet{balandat2020botorch}, which replaces the original acquisition function optimization problem
with a sample average approximation (SAA). The samples used within this SAA are obtained by applying the reparameterization trick \citep{wilson2018maxbo} to the posterior distributions on $f$ and $g$. The approximate computation of $\eiuu$ used within SAA is summarized in Algorithm~\ref{alg:eiuu}. This is implemented in BoTorch~\citep{balandat2020botorch}, and uses $N_g = 8$ samples from $g$ samples and $N_f = 32$ samples from $f$. As is standard for BoTorch AFs, we use quasi-Monte Carlo samples obtained via scrambled Sobol' sequences \citep{owen1998sobol}. Optimization is performed via L-BFGS-B.

\begin{algorithm}
\caption{Computation of qNEIUU}
\begin{algorithmic}[1]
    \label{alg:eiuu}
    \REQUIRE{$X$, a batch of $q$ design points to evaluate (a $q \times d$ matrix);\\
    \enspace\qquad $X_{\mathrm{obs}}$, the set of $n$ previously observed, potentially noisily observed points (a $n \times d$ matrix);\\
\enspace\qquad
$f$, a probabilistic surrogate model fitted on the experimental data $\D$ with $k$ outcomes;\\
    \enspace\qquad $g$, a probabilistic preference model fitted on preference feedback dataset $\mathcal{P}$;\\
    \enspace\qquad$N_{f}$, $N_{g}$, the number of MC samples from $f$ and $g$;
    }\\
    
    \vspace{0.5em}
    \STATE $[\tilde{Y}, \tilde{Y}_{\mathrm{obs}}] \gets \text{Draw } N_{f}~ \text{samples from}  f([x, x_{\mathrm{obs}}])$\\
        \# $\tilde{Y}$ is a tensor of size $N_{f} \times q \times k$ \\
    \# $\tilde{Y}_{\mathrm{obs}}$ is a tensor of size $N_{f} \times p \times k$ 

    \vspace{0.5em}
    \STATE $[\tilde{U}, \tilde{U}_{\mathrm{obs}}] \gets \text{Draw } N_{g}~ \text{samples from } g([\tilde{Y}, \tilde{Y}_{\mathrm{obs}}])$\\
    \# $\tilde{U}$ is a tensor of size $N_{g} \times N_{f} \times q$ \\
    \# $\tilde{U}_{\mathrm{obs}}$ is a tensor of size $N_{g} \times N_{f} \times p$ 
    \vspace{0.5em}
    
    \STATE $U^{*}_{i,j} \gets \max_{\ell=1,\ldots, q}\tilde{U}_{i,j,\ell}$
    \STATE $U^{*}_{\mathrm{obs}\;i,j} \gets \max_{\ell=1,\ldots, n}\tilde{U}_{\mathrm{obs}\;i,j,\ell}$

    \STATE   $\Delta_{i,j} := 
 \{U_{i,j}^{*} - U^{*}_{\mathrm{obs}\;i,j}\}^{+}$\\
    
    \STATE qNEIUU $\gets\tfrac{1}{N_g N_f}\sum_{i=1:N_g}{\sum_{j=1:N_f} \Delta_{i,j}}$\\
    
    \STATE \textbf{return} qNEIUU
\end{algorithmic}
\end{algorithm}

\subsection{Monte Carlo BALD}
We leverage a novel MC implementation of BALD~\citep{Houlsby2011-jf} that allows us to reparameterize the optimization of BALD in the design, rather than outcome space by propagating samples from $f(x)$ through $g$ into in our acquisition function. The computation of this acquisition function is summarized in Algorithm~\ref{alg:bald}. Here, $\Phi$ is the standard normal CDF and $H_b$ is the binary entropy function. This acquisition function is optimized via the standard BoTorch approach described in the previous subsection.

\begin{algorithm}
\caption{MC-BALD with $\Y$ search space}
\begin{algorithmic}[1]
    \label{alg:bald}
    \REQUIRE{$y_1$, $y_2$, a pair of comparison design points to evaluate;\\
    \enspace\qquad $g$, a probabilistic preference model\\
    \enspace\qquad$N_{MC}$, the number of MC samples;
    }\\
    
    \vspace{0.5em}
    \STATE $\mu, \sigma^2 \gets \text{mean and var of } g(y_1) - g(y_2)$
    
    \STATE $z_{\mathrm{posterior}} \gets \Phi(\frac{\mu}{\sqrt{\sigma^2 + 1}})$
    \STATE $h_{\mathrm{posterior}} \gets H_{b}(z_{\mathrm{posterior}})$
    \STATE $s \gets$ Draw $N_{MC}$ samples from $N(\mu, \sigma^2)$
    \STATE $z_{\mathrm{samples}} \gets \Phi(s)$
    \STATE $h_{\mathrm{\mathrm{conditional}}} \gets H_{b}(z_{\mathrm{samples}}).\mathrm{mean}()$
    
    \STATE \textbf{return} $h_{\mathrm{posterior}} - h_{\mathrm{conditional}}$
\end{algorithmic}
\end{algorithm}

\subsection{Uniform Sampling Over $\Y_0$}
For the Uniform Random baseline, we need to empirically determine the  bounds of each outcome.
To do so, we first sample a large number of random points ($10^8$ in this case) in $X \in \X$, and obtain an empirical sample of $Y = \ftrue(X)$ by evaluating the test response function.
Then we are able to identify the empirical lower and upper of $\ftrue$ in $\mathbb{Y}$, denoted as $Y_{\min}$ and $Y_{\max}$ respectively.
Finally, we scale (up or down depending on its sign and whether it's lower or upper bound) $Y_{\min}$ and $Y_{\max}$ by 10\% (or 20\% if one side of the bound is 0) to include additional potential boundary values that are not captured by our sampling scheme. This provides a generous uniform prior over the achievable set when used as a baselines in our experiments.

\section{THEORETICAL RESULTS}
\label{sec:theoretical}
All probabilities and expectations in this section are with respect to the posterior distribution on $g$ given $m$ DM queries. Previously we used the notation $\E_m$, but here we drop the subindex $m$ for brevity.
\subsection{Proof of Theorem 1}
\begin{reptheorem}{thm1}
 
Suppose that $\Y$ is compact, $\mu^g_m(y)$ and $K^g_m(y,y)$ are continuous functions of $y$, and $\lambda=0$. Then,
\begin{equation}
\argmax_{y_1, y_2 \in \Y} \AF(y_1,y_2) \subseteq \argmax_{y_1, y_2 \in \mathbb{R}^k} V(y_1,y_2),
\end{equation}
and the left-hand side is non-empty.
\end{reptheorem}
\begin{proof}
We first observe that the left-hand side is non-empty since  $\Y$ is compact
and $\AF$ is continuous.
Continuity of $\AF$ follows directly from the continuity of $\mu^g_m$ and $K^g_m$ along with \eqref{eq:EUBO-analytic} (note that $\Delta\Phi(\Delta/\sigma) + \sigma\varphi(\Delta/\sigma)$ is continuous at $\sigma=0$).

Recall that
\begin{equation*}
    V(y_1, y_2) = \E\left[\max_{y\in \Y}\E\left[g(y) \mid \left(y_1, y_2, r(y_1, y_2)\right)\right] - \max_{y\in \Y}\E[g(y)]\right].
\end{equation*}
Since $\max_{y\in \Y}\E[g(y)]$ does not depend on $(y_1,y_2)$, we have
\begin{equation*}
   \argmax_{y_1, y_2 \in \mathbb{R}^k} V(y_1,y_2) =  \argmax_{y_1, y_2 \in \mathbb{R}^k} W(y_1,y_2),
\end{equation*}
where 
\begin{equation*}
    W(y_1, y_2) = \E\left[\max_{y\in \Y}\E\left[g(y) \mid \left(y_1, y_2, r(y_1, y_2)\right)\right]\right].
\end{equation*}
Thus, it suffices to show that
\begin{equation}
\argmax_{y_1, y_2 \in \Y} \AF(y_1,y_2) \subseteq \argmax_{y_1, y_2 \in \mathbb{R}^k} W(y_1,y_2).
\label{eq:thm1_equiv}
\end{equation}

To show \eqref{eq:thm1_equiv}, we rely on an idea from \cite{viappiani2010optimal}. For $i\in\{1,2\}$ let
\begin{equation*}
 y^*(y_1,y_2, i) \in \argmax_{y\in \Y} 
\E\left[g(y) \mid (y_1, y_2, i)\right].   
\end{equation*}
Below, we show that 
\begin{equation}
\label{eq:lemma}
\AF(y^*(y_1,y_2, 1), y^*(y_1,y_2,2)) \ge W(y_1,y_2).
\end{equation}

For generic $y_1$, $y_2$, 
we also have 
\begin{equation}
\begin{split}
\AF(y_1, y_2) \le W(y_1, y_2)
\end{split}
\label{eq:lemma2}
\end{equation}
because the right-hand side is bounded below by replacing $\max_y \E\left[g(y) \mid \left(y_1, y_2, r(y_1, y_2)\right)\right]$ by
$\E[g(y_{r(y_1,y_2)})\mid \left(y_1, y_2, r(y_1, y_2)\right)]$ in the definition of $W$, i.e., by using the item that the DM selected in the query as our estimate of the item that the DM most prefers.

\newcommand{\ystar}{\tilde{y}^*}
\newcommand{\y}{\tilde{y}}

With these inequalities established, 
let $(y_1, y_2$) be in the right-hand side of \eqref{eq:thm1_equiv}.
Suppose for contradiction that 
$(y_1, y_2)$ is not in the left-hand side of \eqref{eq:thm1_equiv}. Then there must be some $(\y_1, \y_2)$ such that 
$W(\y_1, \y_2) > W(y_1, y_2)$.

Let $\ystar_i = y^*(\y_1, \y_2,i)$ for $i=1,2$. We have 
\begin{align*}
\AF(\ystar_1, \ystar_2) &\ge 
W(\y_1, \y_2)\\
& >  W(y_1, y_2)\\
& \ge \AF(y_1, y_2)\\ 
& \ge \AF(\ystar_1, \ystar_2).
\end{align*}
The first inequality is due to \eqref{eq:lemma}, which we show below. 
The second is due to our supposition for contradiction that $(\y_1, \y_2)$ had a strictly larger value of $W$ than $(y_1,y_2)$.
The third is due to \eqref{eq:lemma2}.
The fourth is because $(y_1$, $y_2)$ was chosen to maximize $\AF$ over $\Y\times\Y$ and because 
$\ystar_i \in \Y$ for both $i=1,2$.

This is a contradiction.  Thus, it must be that $(y_1, y_2)$ is also in the left-hand side of \eqref{eq:thm1_equiv}.

It now remains to show that \eqref{eq:lemma} holds. To this end, let $y^*_i = y^*(y_1,y_2,i)$.
Observe that 
\begin{align*}
W(y_1,y_2)
 &= \sum_i \prob(r(y_1,y_2)=i) \E[g(y^*_i) | (y_1,y_2, i)]
\end{align*}
and $g(y^*_i) \leq\max\{g(y^*_1), g(y^*_2)\}$ for both $i=1,2$. Thus,
\begin{align*}
W(y_1,y_2)
= & \sum_i \prob(r(y_1, y_2)=i) \E[g(y^*_i) | (y_1,y_2, i)]\\
\le {} &
\sum_i \prob(r(y_1,y_2)=i) \E[\max\{g(y^*_1), g(y^*_2)\} | (y_1,y_2, i)] \\
= {} & \E[\max\{g(y^*_1), g(y^*_2)\}]\\
= {} & \AF(y_1^*, y_2^*).
\hfill\qedhere
\end{align*}

\subsection{Proof of Theorem 2}
Here we prove Theorem 2. First we introduce additional notation and prove several auxiliary lemmas.

\begin{definition}
We define 
\begin{equation*}
 \AF_\lambda(y_1, y_2) = \E[g(y_{r(y_1,y_2)})],   
\end{equation*}
 where $r(y_1,y_2)$ is a random variable whose conditional distribution given $g(y_1)$ and $g(y_2)$ is given by 
 \begin{equation*}
     \prob(r(y_1, y_2) = 1 \mid g(y_1), g(y_2)) = \Phi\left(\frac{g(y_1)-g(y_2)}{\sqrt{2}\lambda}\right),
 \end{equation*}
 and $\prob(r(y_1,y_2) = 2 \mid g(y_1), g(y_2)) = 1 - \prob(r(y_1,y_2) = 1 \mid g(y_1), g(y_2)) $.
 
 Similarly, we also define
\begin{equation*}
    V_\lambda(y_1, y_2) = \E\left[\max_{y\in \Y}\E\left[g(y) \mid \left(y_1, y_2, r(y_1, y_2)\right)\right] - \max_{y\in \Y}\E[g(y)]\right],
\end{equation*}
and
\begin{equation*}
W_\lambda(y_1, y_2) = \E\left[\max_{y\in \Y}\E[g(y) \mid \left(y_1, y_2, r(y_1, y_2)\right)]\right]    
\end{equation*}
\end{definition}

The above definitions generalize the definitions of $\AF$, $V$, and $W$ to the case where the DM's responses are corrupted by probit noise.

The following inequality is key in our proof of Theorem 2.

\begin{lemma}
\label{lemma:probit_bound}
For any fixed  $r_1, r_2 \in \R$ and $\lambda > 0$,
\begin{equation*}
     \Phi\left(\frac{r_1 - r_2}{\sqrt{2}\lambda}\right) r_1 +\Phi\left(\frac{r_2 - r_1}{\sqrt{2}\lambda}\right) r_2 \geq \max\{r_1, r_2\} - \lambda C,   
\end{equation*}
where $C=e^{-1/2}/\sqrt{2}$.
\end{lemma}
\begin{proof}
Without loss of generality we may assume that $r_1 \geq r_2$. Then, we want to  prove that
\begin{equation*}
  \Phi\left(\frac{r_1 - r_2}{\sqrt{2}\lambda}\right) r_1 +\Phi\left(\frac{r_2 - r_1}{\sqrt{2}\lambda}\right) r_2 \geq r_1 - \lambda C,
\end{equation*}
By recalling that $\Phi(t) = 1 - \Phi(-t)$ for any $t\in\R$, the inequality above can be further rewritten as
\begin{equation*}
 \Phi\left(\frac{r_1 - r_2}{\sqrt{2}\lambda}\right) r_1 +\left[1 -\Phi\left(\frac{r_1 - r_2}{\sqrt{2}\lambda}\right)\right] r_2 \geq r_1 - \lambda C.   
\end{equation*}
Arranging terms and letting $s = (r_1 - r_2)/\sqrt{2}\lambda$, we obtain the equivalent inequality
\begin{equation*}
    \frac{C}{\sqrt{2}} \geq s\left[1 - \Phi(s)\right].
\end{equation*}
Thus, it suffices to show that the above inequality holds for any $s\geq 0$. To this end, recall that
\begin{equation*}
 \Phi(s) = \frac{1}{2}\left[1 + \mathrm{erf}\left(\frac{s}{2}\right)\right], 
\end{equation*}
where $\mathrm{erf}$ is the Gauss error function. Thus, the above inequality is equivalent to
\begin{equation*}
    \frac{C}{\sqrt{2}} \geq \frac{s}{2}\mathrm{erfc}\left(\frac{s}{2}\right),
\end{equation*}
where $\mathrm{erfc} = 1 - \mathrm{erf}$ is the complementary Gaussian error function. Using the well-known inequality $e^{-t^2} \geq \mathrm{erfc}(t)$, which is valid for all $t > 0$ (see, e.g., \cite{chang2011chernoff}), it suffices to show that
\begin{equation*}
    \frac{C}{\sqrt{2}} \geq \frac{s}{2} e^{-\frac{s^2}{2}};
\end{equation*}
i.e., 
\begin{equation*}
    \frac{1}{2}e^{-\frac{1}{2}} \geq \frac{s}{2} e^{-\frac{s^2}{2}}.
\end{equation*}
The above inequality can be easily verified by noting that the right-hand side reaches its maximum value at $s=1$.
\end{proof}

The following result shows that the expected utility of a DM expressing responses corrupted by probit noise with parameter $\lambda$ is close to that of a DM expressing noiseless responses.
\begin{lemma}
\label{lemma:eubo_probit_bound}
The following inequality holds for any $y_1,y_2\in \Y$ and $\lambda > 0$:
\begin{equation*}
    \AF_\lambda(y_1, y_2) \geq \AF_0(y_1, y_2) - \lambda C,
\end{equation*}
where $C$ is defined like in Lemma~\ref{lemma:probit_bound}.
\end{lemma}
\begin{proof}
Note that
\begin{align*}
 \E[g(y_{r(y_1, y_2)})\mid g(y_1), g(y_2)] = {} & \Phi\left(\frac{g(y_1) - g(y_2)}{\sqrt{2}\lambda}\right) g(y_1) +
 \\ 
 & \Phi\left(\frac{g(y_2) - g(y_1)}{\sqrt{2}\lambda}\right) g(y_2).  
\end{align*}
Hence, from Lemma~\ref{lemma:probit_bound} it follows that
\begin{equation*}
\E[g(y_{r(y_1, y_2)})\mid g(y_1), g(y_2)] \geq \max\{g(y_1),g(y_2)\} - \lambda C.    
\end{equation*}
The desired result can now be obtained by taking expectations over $g(y_1)$ and $g(y_2)$ on both sides of the above inequality.
\end{proof}

Our last lemma simply shows that the function $W_\lambda$ dominates $\AF_\lambda$ for any $\lambda > 0$.

\begin{lemma}
\label{lemma:w_geq_eubo}
The following inequality holds for any $y_1,y_2\in \Y$ and $\lambda > 0$:
\begin{equation*}
 W_\lambda(y_1, y_2) \geq \AF_\lambda(y_1, y_2).   
\end{equation*}
\end{lemma}
\begin{proof}
We have
\begin{align*}
    W_\lambda(y_1, y_2) &= \E\left[\max_{y\in \Y}\E[g(y) \mid r(y_1, y_2)]\right]\\
    & \geq \E\left[\E[g(y_{r(y_1, y_2)}) \mid r(y_1, y_2)]\right]\\
    &= \E[g(y_{r(y_1, y_2)})]\\
    &= \AF_\lambda(y_1, y_2).
\end{align*}
\end{proof}

We are now in position to prove Theorem 2.

\begin{reptheorem}{thm2}
Let $(y_1^*, y_2^*) \in \argmax_{y\in \Y} \AF(y_1,y_2)$. Then,
\begin{equation*}
 V_\lambda(y_1^*, y_2^*) \geq \max_{y_1,y_2\in \Y}V_0(y_1, y_2) - \lambda C,  
\end{equation*}
where $C$ is defined like in Lemma~\ref{lemma:probit_bound}.
\end{reptheorem}
\begin{proof}
We have
\begin{align*}
W_\lambda(y_1^*, y_2^*) &\geq  \AF_\lambda(y_1^*, y_2^*)\\
&\geq \AF_0(y_1^*, y_2^*) - \lambda C\\
& = \max_{y\in \Y} \AF_0(y_1,y_2) - \lambda C\\
& = \max_{y\in \Y} W_0(y_1,y_2) - \lambda C,
\end{align*}
where the first line follows from Lemma~\ref{lemma:w_geq_eubo}, the second line follows from Lemma~\ref{lemma:eubo_probit_bound}, the third line follows from the definition of $(y_1^*, y_2^*)$, and the fourth line follows from Theorem~\ref{thm1}. The desired result can now be obtained by subtracting $\max_{y\in \Y}\E[g(y)]$ from both sides of the inequality.
\end{proof}

\subsubsection{Closed Form Expression of EUBO Under a Gaussian Posterior}
\label{sec:analytic_eubo}
To compute $\AF$ in closed form, recall that $\mu_m^g$ and $K_m^g$ are the mean and posterior covariance of the approximate GP posterior after $m$ queries. We rewrite  $\AF$ as
\begin{align*}
\AF(y_1,y_2) &= \E[\{g(y_1) - g(y_2)\}^+ + g(y_2)]\\
&= \E[\{g(y_1) - g(y_2)\}^+] + \mu_m^g(y_2).
\end{align*}

Now let $\Delta(y_1,y_2)$ and $\sigma^2(y_1,y_2)$ be the mean and variance of $g(y_1) - g(y_2)$:
\begin{equation*}
\Delta(y_1,y_2) 
= \mathbb{E}[g(y_1) - g(y_2)]
= \mu_m^g(y_1) - \mu_m^g(y_2)
\end{equation*}
and 
\begin{align*}
\sigma^2(y_1, y_2) 
&= \mathrm{Var}[g(y_1) - g(y_2)]\\
&= K_m^g(y_1,y_1) + K_m^g(y_2,y_2) - 2K_m^g(y_1,y_2).
\end{align*}
Using the standard formula for the expectation of the positive part of a normal random variable with a given mean and variance, and  dropping $y_1,y_2$ from the arguments to $\Delta$ and $\sigma$ for brevity, we get
\begin{equation*}
   \E[\{g(y_1) - g(y_2)\}^+] = \Delta\Phi\left(\frac{\Delta}{\sigma}\right) +
\sigma\varphi\left(\frac{\Delta}{\sigma}\right),
\end{equation*}
and thus
\begin{equation}
\label{eq:EUBO-analytic}
\AF(y_1,y_2) = 
\Delta\Phi\left(\frac{\Delta}{\sigma}\right) +
\sigma\varphi\left(\frac{\Delta}{\sigma}\right) + 
 \mu_m^g(y_2),
\end{equation}
where $\varphi$ and $\Phi$ are the standard normal PDF and CDF, respectively.

\end{proof}

\section{SIMULATION DESIGN AND IMPLEMENTATION DETAILS}
\label{sec:sim_design_detail}
\subsection{Choice of Batch Sizes and Number of DM Queries}
Our simulation experiments are configured to mimic aspects of real-world experiments in which DMs may wish to perform Bayesian optimization using preference models.  To do this, we rely on pilot studies related to early versions of the PE strategies developed in this work. Participants in the pilots were data scientists and ML engineers at Meta who routinely used Bayesian optimization to tune recommender system ranking policies. Standard A/B tests consider a large number of number of initial  design points that are on the order of 3-5x the number of input dimensions, which motivates the initial batch sizes used in our simulation experiments. Subsequent experiments tend to use approximately half the number of points. To decide on what a sensible number of preference queries would be, we analyzed data from the pilot study to find that participants spent on average 7.8-12.7 seconds to perform pairwise comparisons between between problems with fairly large numbers of outputs (see Table~\ref{tab:pilot_responses}).
Based on this, we would estimate that it would conservatively take DMs around 15 minutes to perform 75 comparisons, or 5 minutes per PE stage if split across three rounds of experimentation.  Finally, we observed that the empirical error rate was approximately 10\%, and so we used this value for our DM noise model in the MT.

\begin{table}[h]
 \centering
 \begin{tabular}{c|c|c|c|c}
 Pilot & $d$ & $k$ & Response mean (s) & Response sd (s)\\
 \midrule
     1 &     11 &      9 & 10.7 & 5.7  \\
     2 &      8 &      9 & 7.8 & 3.0  \\
     3 &      4 &      4 & 8.6 & 6.0  \\
     4 &      8 &      6 & 12.7 & 8.1 \\
     5 &      8 &      6  & 8.8 & 7.2 \\
 \end{tabular}
 \caption{
 Summary statistics of input dimensionality ($d$), number of outcomes ($k$), and response times of data scientists in seconds to pairwise preference learning comparisons from pilot five studies.}
 \label{tab:pilot_responses}
\end{table}

\section{ADDITIONAL SIMULATION RESULTS}
\label{sec:additional_sims}
\subsection{Additional Test Problems}

We consider additional simulation environments here.
There are four outcome and utility function combinations presented in the main paper and additional four in the supplementary material, totalling eight simulation environments. These include surrogates of real-world simulators (Vehicle Safety and Car Cab Design) and widely used synthetic functions (OSY and DTLZ2), alongside various utility functions.
These test problems have 5-8 input dimensions ($d$), and 4-9 outcomes ($k$).
Table~\ref{tab:complete_test_func} shows the full list of simulation environments we used in this paper.
\S\ref{sec:test_functions} describe those outcome and utility in detail.

\subsection{Identifying High-Utility Designs with Preference Exploration}
Figure~\ref{fig:sm_within_batch_sim} Shows the experiment results for identifying high-utility designs with PE for all suite of test problems.
While most results hold similar to what we observe in the main text, EUBO-$\zeta$ and EUBO-$\tf$ are performing dramatically better compared to other methods in the OSY with exponential function sum with sigmoid constraints problem.
This is possibly because of the unique characteristic of this problem where many points in $\Y$ are violating the constraints and result in near-zero values.
Comparing those near-zero utility values, while potentially helpful for learning the shape of utility functions within the achievable region, is not necessarily contributing much to the identification of the maximizer of $\gtrue(\ftrue(\X))$.

\begin{figure}[h]
    \centering
    \includegraphics[width=1\textwidth]{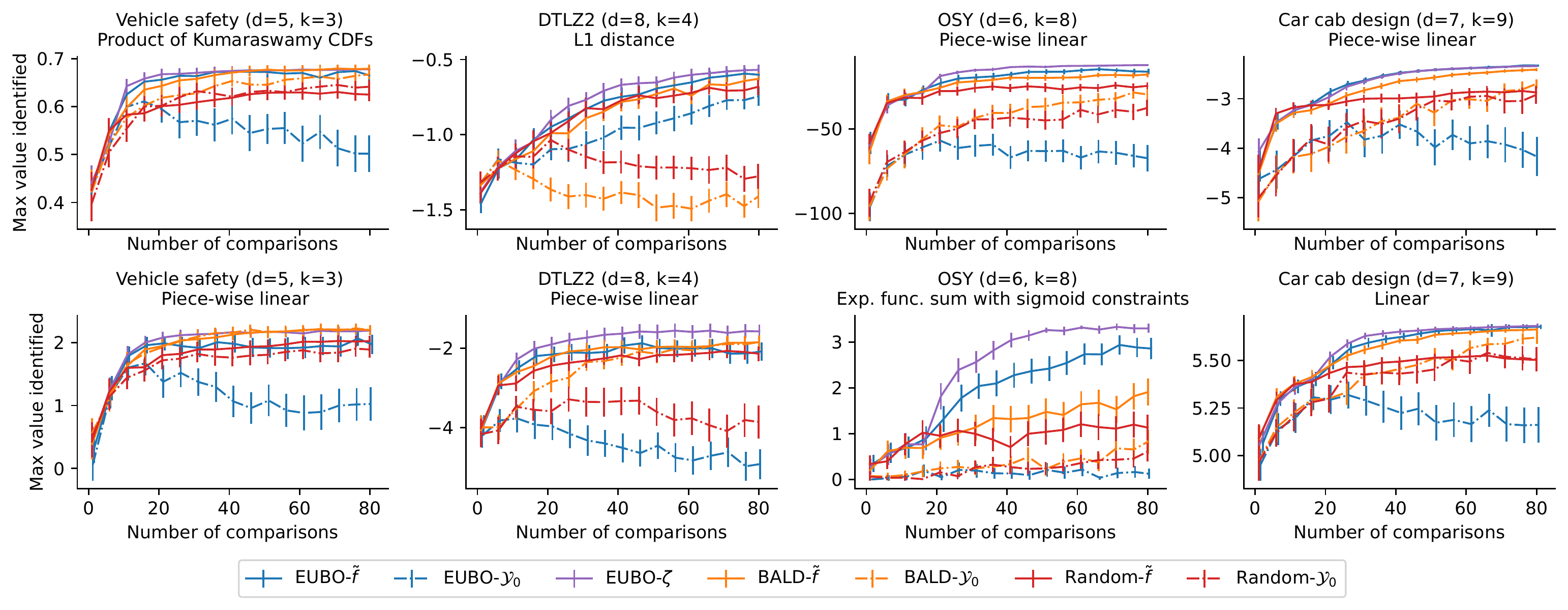}
    \caption{
    Simulation results for identifying high-utility designs with PE for all suite of test problems.
    Plotted are values of the maximum posterior predictive mean after a given number of pairwise comparisons during the first stage of preference exploration.
    }
    \label{fig:sm_within_batch_sim}
\end{figure}

\subsection{Preference Exploration with Multiple PE Stages }
 BOPE with multiple stages for additional test problems are plotted in Figure~\ref{fig:all_multi_batch_sim}. Figure~\ref{fig:all_multi_batch_sim} shows the results of BOPE with multiple PE stages, and the results largely align what we observe in the main text where EUBO-$\zeta$ and EUBO-$\tf$ consistently performing well across all problems.

\begin{figure}
    \centering
    \includegraphics[width=1\textwidth]{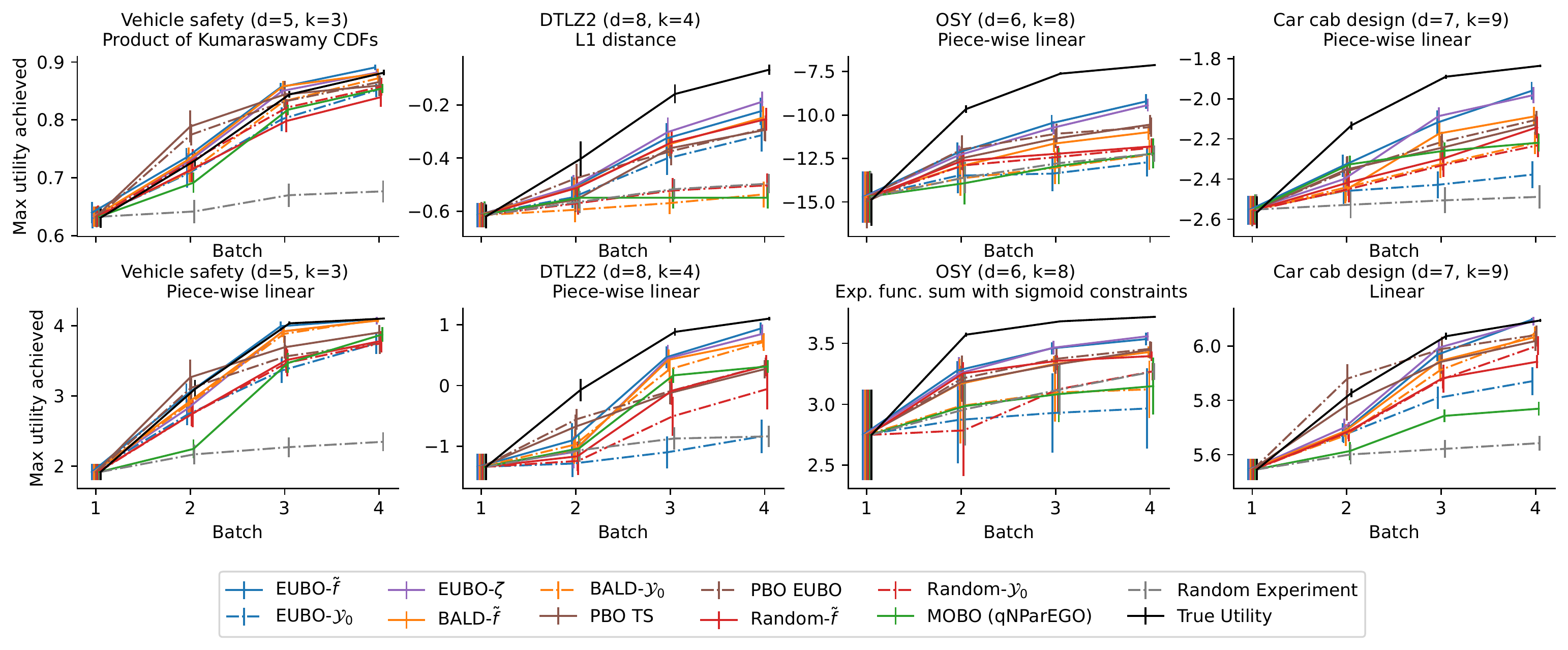}
    \caption{BOPE performance with multiple \PE{} stages for all benchmark problems.
    ``True Utility'', the grey line on each plot, represents an approximate upper bound on the performance achievable via Bayesian optimization with a known utility function, blue and orange show our proposed methods, and other lines show competing baseline methods.
    The top row includes outcome and utility functions whose maximum utility achieved are plotted in the main paper in Figure~\ref{fig:multi_batch_sim} and the bottom row shows maximum utility achieved for the additional set of outcome and utility functions.
    }
    \label{fig:all_multi_batch_sim} 
\end{figure}

\subsection{BOPE with One Preference Exploration Stage}

Figure~\ref{fig:interactive_oneshot_diff_in_util} directly shows the differences in final outcomes achieved by each method by comparing these two PE schemes.  In addition to the insights from the main text, one can also see that PE AFs that perform search of $\Y_0$, rather than based on the posterior of $f$, tend to perform much better when all learning occurs upfront.  This makes intuitive sense, since $\Y_0$ is agnostic to any additional updated surrogate models $f_n$ collected throughout the experimentation process. 

Finally, Figure~\ref{fig:oneshot_multi_batch_sim} shows the full optimization trajectory for the single \PE{} stage case.  From these plots it is apparent that PBO-based strategies make little  progress in the optimization goal after the second stage.

\begin{figure}
\centering
    \centering
    \includegraphics[width=\textwidth]{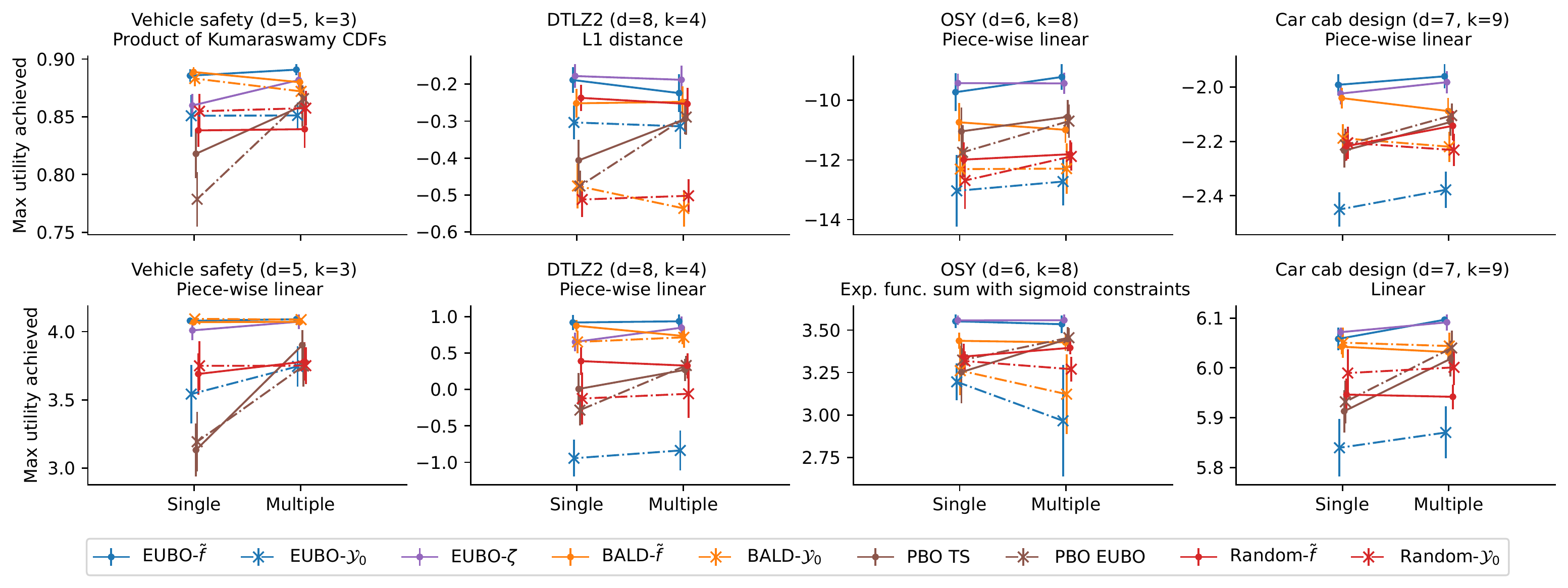}
    \caption{
    Maximum utility achieved after the last batch for BOPE with single PE stage and multiple PE stages for all benchmark problems.
    }
    \label{fig:interactive_oneshot_diff_in_util}
\end{figure}

\subsection{Probit Comparison Noise}
\label{sec:probit_sims}
In all simulation studies we have presented so far, we have been considering a constant 10\% error rate.  However, it is plausible that DMs may make errors in ways that vary with the utilities. Here, we consider the case where DMs are more likely to make mistakes when utilities have similar values, and study the behavior of the optimization strategies when such noise is present. 

Concretely, we use the probit likelihood with noise level $\lambda > 0$ introduced in \S\ref{sec:models}. Thus, when the DM is presented with a query  $(y_1, y_2)$,  $y_1$  is chosen with probability $\Phi\left(\frac{g(y_1)-g(y_2)}{\sqrt{2}\lambda}\right)$ and $y_2$ is chosen with probability $1 - \Phi\left(\frac{g(y_1)-g(y_2)}{\sqrt{2}\lambda}\right)$. For each experiment we set $\lambda$ so that  the DM makes a comparison error with probability 0.1 on average when presented with queries constituted by outcomes of designs that are in the top decile of $\gtrue(\ftrue(\X))$.
Table~\ref{tab:probit_error_rate} shows the average empirical comparison error rates observed under such probit noises.

Figure~\ref{fig:probit_comp_noise} shows optimization performance for BOPE with multiple PE stages for the four test problems presented in the main text using EUBO-$\tf$,  EUBO-$\zeta$, BALD-$\tf$, PBO TS, and PBO EUBO.
Similar to what we observe in the main text, EUBO-$\zeta$ and EUBO-$\tf$ perform the same or better than all other PE and PBO baselines considered.

\begin{figure}[h]
    \centering
    \includegraphics[width=\textwidth]{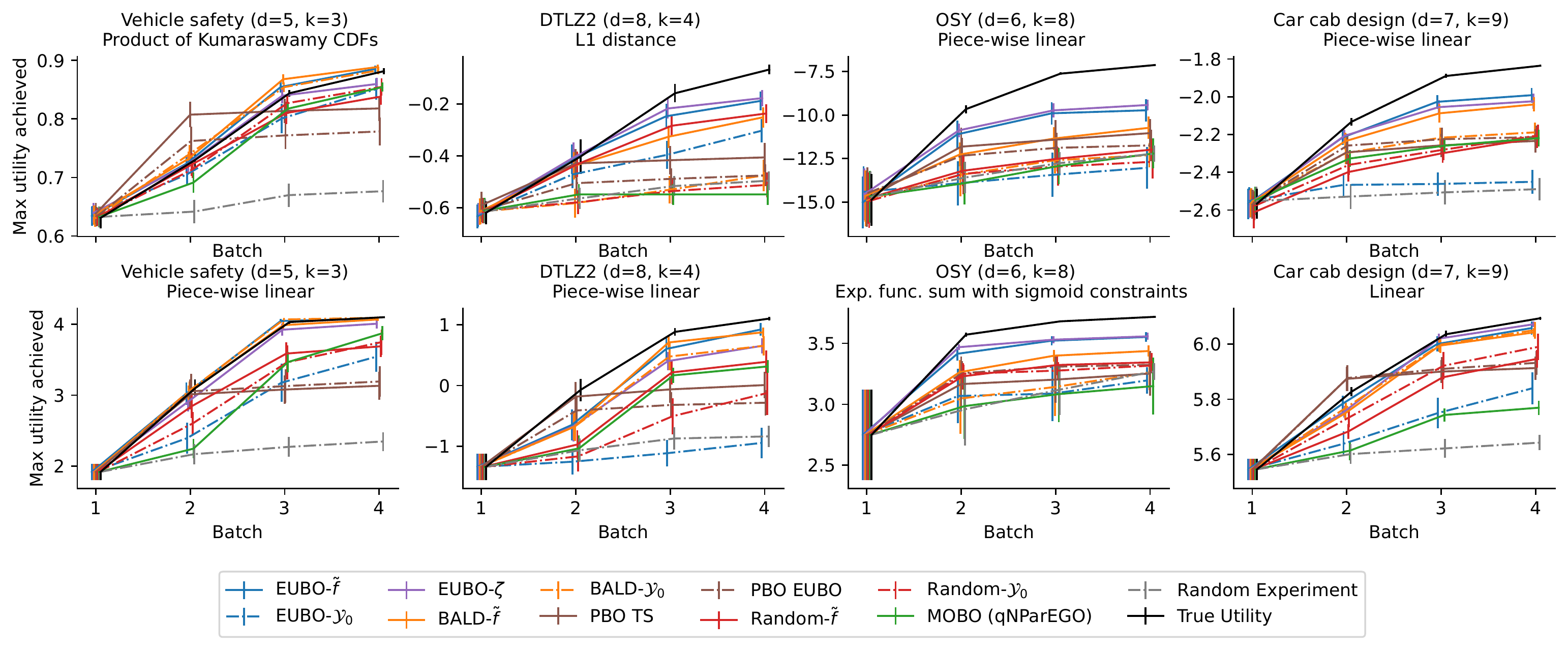}
    \caption{
    BOPE with a single \PE{} stage full optimization trajectory for all benchmark problems.
    }
    \label{fig:oneshot_multi_batch_sim}
\end{figure}

\begin{figure}[h]
    \centering
    \includegraphics[width=1\textwidth]{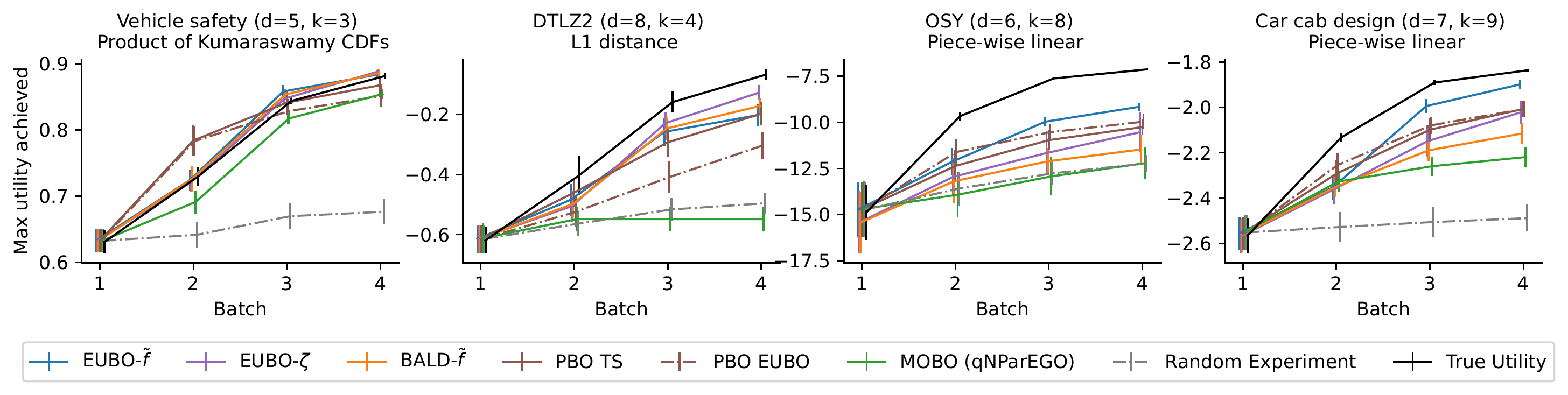}
    \caption{
    BOPE with multiple PE stages under probit comparison noise.
    }
    \label{fig:probit_comp_noise}
\end{figure}

\begin{table}[h]
\begin{tabular}{l|rrrrr}
\toprule
Test Problem & BALD-$\tf$ & EUBO-$\tf$ & EUBO-$\zeta$ & PBO EUBO & PBO TS \\
\midrule
Vehicle safety, product of Kumaraswamy CDFs &            12.8\% &             5.1\% &        10.0\% &     7.9\% &   7.8\% \\
DTLZ2, L1 distance &              3.0\% &             7.0\% &         6.9\% &     9.1\% &   4.1\% \\
OSY, piece-wise linear                &             11.3\% &            14.0\% &        20.1\% &    15.8\% &   8.0\% \\
Car cab design, piece-wise linear     &            12.1\% &             8.5\% &        17.9\% &    10.2\% &   7.3\% \\
\bottomrule
\end{tabular}
    \caption{Empirical comparison error rate under probit noise levels considered in this subsection.%
    }
    \label{tab:probit_error_rate}
\end{table}

\newpage

\section{OUTCOME AND UTILITY FUNCTIONS}
\label{sec:test_functions}

In this section, we describe all outcome and utility functions used in our simulation studies.

\subsection{Outcome Functions}
\subsubsection{DLTZ2}
DTLZ2 function was first introduced by \citet{deb2005scalable}, allowing for arbitrary input dimension $d$ and output dimension $k$ subject to $d > k$.
$\mathbb{X} = [0, 1]^d$.
For a DLTZ2 function $f$ with $k$-dimensional output, we have:

\begin{align*}
m &= d - k + 1\\
g(x) &= \sum_{i=m}^{d-1} (x_i - 0.5)^2\\
f_{j}(x) &= -(1+g(x))\left(\prod_{i=1}^{k-j-1} \cos \left(\frac{\pi}{2} x_{i}\right)\right)\cdot \\
&\quad\quad\mathbb{1}_{j>1}\sin \left(\frac{\pi}{2} x_{k-j-1}\right)
\end{align*}

\subsubsection{Vehicle Safety}
This a test problem for optimizing vehicle crash-worthiness with $d=5$ and $k=3$.
$\mathbb{X} = [1, 3]^5$.
We refer the readers to \citet{tanabe2020easy, liao2008multiobjective} for details on function definition.
During the simulation, we normalize each component of $f$ to lie between 0 and 1.

\subsubsection{Car Cab Design}
We refer the readers to \citet{deb2013evolutionary, tanabe2020easy} for details.
Note that in the original problem, there are stochastic components which we exclude in the experiments to obtain deterministic ground-truth outcome function.
During the simulation, we normalize each dimension of $f$ to between 0 and 1.

\subsubsection{OSY}
We adapted the constrained optimization OSY problem~\citep{osyczka1995new} into a multi-objective problem by treating all constraints as objectives.
We additionally flipped the signs of the two objectives of OSY such that all outputs are intended to be maximized. This adaptation makes OSY to be an outcome function with 6-dimensional inputs and 8-dimensional outputs.

\subsection{Utility Functions}
We consider several utility functions to capture several types of ways in which DMs may weigh the observed objective values. For all outcome functions, we consider piece-wise linear functions to represent constraint-like behavior and decreasing marginal returns.
We then designed four separate utility functions by taking the characteristics of each individual test outcome function.
Here we describe them in detail.

\subsubsection{Piecewise Linear Function}
We performed experiments on all test outcome functions using piece-wise linear functions as their shapes correspond to real-world diminishing marginal returns on outcomes and sharp declines in utility once constraints are violated.
For a $k$-dimensional input vector $\mathbf{y}$, this utility function $g$ is defined as
$$
g(\mathbf{y}; \mathbf{\beta_1}, \mathbf{\beta_2}, \mathbf{t}) = \sum_{i=1}^{k} h_{i}\left(y_{i}\right)
$$
where
\begin{align*}
    h_{i}\left(y_{i}\right) = \begin{cases} 
      \beta_{1, i} \cdot y_i + (\beta_{2, i} - \beta_{1, i})t_i & y_i < t_i\\
      \beta_{2, i} \cdot y_i & y_i\geq t_i
   \end{cases}
\end{align*}

For DTLZ2 (d=8, k=4) problem, we set 
\begin{align*}
\mathbf{\beta_1} &= [4, 3, 2, 1]\\
\mathbf{\beta_2} &= [0.4, 0.3, 0.2, 0.1]\\
\mathbf{t} &= [1, 1, 1, 1].
\end{align*}

For vehicle safety problem, we set 
\begin{align*}
\mathbf{\beta_1} &= [2, 6, 8]\\
\mathbf{\beta_2} &= [1, 2, 2]\\
\mathbf{t} &= [0.5, 0.8, 0.8].
\end{align*}

For the car cab design problem, we set 
\begin{align*}
\mathbf{\beta_1} &= [7.0, 6.75, 6.5, 6.25, 6.0, 5.75, 5.5, 5.25, 5.0]\\
\mathbf{\beta_2} &= [0.5, 0.4, 0.375, 0.35, 0.325, 0.3, 0.275, 0.25, 0.225]\\
\mathbf{t} &= [0.55, 0.54, 0.53, 0.52, 0.51, 0.5, 0.49, 0.48, 0.47]
\end{align*}
and the threshold parameter $\mathbf{t_i} = 0.75$ for all $i$.

For the OSY problem, we set
\begin{align*}
\mathbf{\beta_1} &= [0.02, 0.2, 10, 10, 10, 10, 10, 10]\\
\mathbf{\beta_2} &= [0.01, 0.1, 0.1, 0.1, 0.1, 0.1, 0.1, 0.1]\\
\mathbf{t} &= [1000, -100, 0, 0, 0, 0, 0, 0].
\end{align*}

\subsubsection{Linear Function}
For the car cab design problem, we experiment with a linear utility function.
For a $k$-dimensional outcomes vector $\mathbf{y}$, this utility function $g$ is defined as
$$
g(\mathbf{y}; \mathbf{\beta}) = \mathbf{\beta}^T\mathbf{y}
$$

Specifically, we set
$$
\beta = [2.25, 2, 1.75, 1.5, 1.25, 1, 0.75, 0.5, 0.25].
$$

\subsubsection{Product of Kumaraswamy Distribution CDFs}
Prior work on preference learning has utilized the Beta CDF to form utility functions \citep{Dewancker2016-ix}.
The Beta CDF provides a convenient, bounded monotonic transform that smoothly varies between increasing and decreasing marginal gains with respect to their inputs.
In this work, we utilize the Kumaraswamy CDF~\citep{kumaraswamy1980generalized, jones2009kumaraswamy}, which behaves much like the Beta CDF, but is differentiable. 
This allows us to optimize qNEIUU via gradient ascent using the true utility function in Section~\ref{sec:sim_study}.
For the vehicle safety problem, we experiment with the product of Kumaraswamy distribution CDFs as its utility function, representing that we wish to simultaneously achieve high utility values along each individual dimension without sacrificing much on others.
For a $k$-dimensional input vector $\mathbf{y}$, this utility function $g$ is defined as
$$
g(\mathbf{y}; \mathbf{a}, \mathbf{b}) = \prod_{i=1}^{k} F_{i}\left(y_{i}\right),
$$
where
$F_i(\cdot)$ is the CDF of a Kumaraswamy distribution with shape parameters $a_i$ and $b_i$.

Specifically, we set 
\begin{align*}
\mathbf{a} &= [0.5, 1, 1.5]\\
\mathbf{b} &= [1, 2, 3].
\end{align*}

\subsubsection{L1 Distance}
For the DTLZ2 problem, we additionally test a utility function using negative L1 distance from a Pareto-optimal point.
This choice of utility function mimics the scenario where the DM wish to keep the outcomes close to a specific desirable state such as physiological measurements in medical applications.

For the DTLZ2 problem, we choose the Pareto-optimal point to be $y^{\mathrm{PO}} = \textnormal{DTLZ2}(x^{\mathrm{PO}})$ where $x^{\mathrm{PO}}_i = 0.5$ for all $i$.

\subsubsection{Exponential Function Sum with Sigmoid Constraints}
The OSY problem is originally a constrained optimization problem and we wish design an utility function that can reflect this nature of this problem.

For a $k$-dimensional outcomes vector $\mathbf{y}$, we recall that the first two dimensions $y_1$ and $y_2$ are objectives we wish to maximize and in the original OSY problems, the remaining six constraints $y_{3..8}$ need to be kept positive.
To reflect these goals, we first normalize the outputs of $y_1$ and $y_2$ to be between 0 and 1 using min-max normalization.
Given an outcome vector $\mathbf{y}$, the utility function is then given by
$$
g(y) = (\exp(y_1) + \exp(y_2))\prod_{j=3}^8 S\left(\frac{50 y_j}{ \min\{-y^{\mathrm{min}}_j, y^{\mathrm{max}}_j\}}\right)
$$
where $S$ is the sigmoid function; $y^{min}_j$ and $y^{max}_j$ are empirically determined lower and upper bound of $y_j$.

\begin{table}[h]
    \resizebox{1\textwidth}{!}{
    \begin{tabular}{c|l|c|c|r}
    \textbf{Outcome function} & \textbf{Utility function} & $d$ & $k$ & \textbf{Reference}\\
    \midrule
    \multirow{2}{*}{Vehicle safety} & Piecewise linear function & \multirow{2}{*}{5} & \multirow{2}{*}{3} & \citet{liao2008multiobjective}\\
    \cline{2-2}
    & Product of Kumaraswamy distribution CDFs & & & \citet{tanabe2020easy}\\
    \hline
    \multirow{2}{*}{DTLZ2} & Piecewise linear function & \multirow{2}{*}{8} & \multirow{2}{*}{4} & \\
    \cline{2-2}
    & L1 distance from Pareto-optimal point & & & \citet{molga2005test}\\
    \hline
    \multirow{2}{*}{OSY} & Piecewise linear function & \multirow{2}{*}{6} & \multirow{2}{*}{8} & \\
    \cline{2-2}
    & Exponential function sum with sigmoid constraints & & & \citet{osyczka1995new}\\    \hline
    \multirow{2}{*} {Car cab design} & Piecewise linear function & \multirow{2}{*}{7} & \multirow{2}{*}{9} & \citet{deb2013evolutionary} \\
    \cline{2-2}
    & Linear function & & & \citet{tanabe2020easy} \\
    \hline
    \end{tabular}
    }
    \caption{
    Complete list of outcome and utility function combinations.
    $d$ and $k$ refers to input and output dimension of the test outcome function respectively.
    }
    \label{tab:complete_test_func}
\end{table}

\end{document}